\documentclass{article}

\usepackage{arxiv}

\usepackage[utf8]{inputenc} 
\usepackage[T1]{fontenc}    
\usepackage[hidelinks]{hyperref}   
\usepackage{url}            
\usepackage{booktabs}       
\usepackage{amsmath, amsfonts, amsthm, amssymb, mathtools, xparse, stackengine}
\usepackage{nicefrac}       
\usepackage{microtype}      
\usepackage{lipsum}		
\usepackage{graphicx}
\usepackage{natbib}
\usepackage{doi}
\usepackage{mathtools, stackengine, threeparttable}
\usepackage{subfigure}
\usepackage{algorithm,algorithmic}
\usepackage{appendix}
\usepackage[utf8]{inputenc}
\usepackage[dvipsnames]{xcolor}
\usepackage{listings}
\usepackage{footmisc}
\usepackage[figuresright]{rotating}
\usepackage{dblfloatfix}

\newtheorem{theorem}{Theorem}
\newtheorem{lemma}{Lemma}

\newtheorem{remark}{Remark}
\newtheorem{corollary}{Corollary}
\newtheorem{observation}{Observation}

\theoremstyle{definition}

\definecolor{custom_blue}{HTML}{1F77B4}
\definecolor{custom_pink}{HTML}{E377C2}
\definecolor{custom_orange}{HTML}{FF7F0E}
\definecolor{custom_purple}{HTML}{9467BD}
\definecolor{custom_green}{HTML}{2CA02C}
\definecolor{custom_red}{HTML}{D62728}
\definecolor{custom_brown}{HTML}{8C564B}

\usepackage{tikz}
\usepackage{pgfplots}
\pgfplotsset{compat=1.17}

\newcommand{\blue}{\raisebox{2pt}{\tikz{\draw[custom_blue, solid, line width=2.3pt](0,0) -- (5mm,0);}}}
\newcommand{\pink}{\raisebox{2pt}{\tikz{\draw[custom_pink, solid, line width=2.3pt](0,0) -- (5mm,0);}}}
\newcommand{\orange}{\raisebox{2pt}{\tikz{\draw[custom_orange, solid, line width=2.3pt](0,0) -- (5mm,0);}}}
\newcommand{\purple}{\raisebox{2pt}{\tikz{\draw[custom_purple, solid, line width=2.3pt](0,0) -- (5mm,0);}}}
\newcommand{\green}{\raisebox{2pt}{\tikz{\draw[custom_green, solid, line width=2.3pt](0,0) -- (5mm,0);}}}

\newcommand{\brown}{\raisebox{2pt}{\tikz{\draw[custom_brown, solid, line width=2.3pt](0,0) -- (5mm,0);}}}

\newcommand*\rot{\rotatebox{90}}
\newcommand{\orcid}[1]{\href{https://orcid.org/#1}{\textcolor[HTML]{A6CE39}{\aiOrcid}}}

\title{Actor Prioritized Experience Replay}

\date{} 					

\author{Baturay Saglam, Furkan B.~Mutlu, Dogan C.~Cicek, Suleyman S.~Kozat \\
	Department of Electrical and Electronics Engineering\\
	Bilkent University, 06800 Bilkent, Ankara, Turkey \\
	\texttt{\{baturay,burak.mutlu,cicek,kozat\}@ee.bilkent.edu.tr} \\
}

\renewcommand{\headeright}{}
\renewcommand{\undertitle}{}
\renewcommand{\shorttitle}{Actor Prioritized Experience Replay}

\hypersetup{
pdftitle={Actor Prioritized Experience Replay},
pdfsubject={arxiv},
pdfauthor={Baturay Saglam, Furkan B.~Mutlu, Dogan C.~Cicek, Suleyman S.~Kozat},
pdfkeywords={deep reinforcement learning, off-policy learning, prioritized experience replay, actor-critic algorithms},
}

\begin{document}
\maketitle

\begin{abstract}
    A widely-studied deep reinforcement learning (RL) technique known as Prioritized Experience Replay (PER) allows agents to learn from transitions sampled with non-uniform probability proportional to their temporal-difference (TD) error. Although it has been shown that PER is one of the most crucial components for the overall performance of deep RL methods in discrete action domains, many empirical studies indicate that it considerably underperforms actor-critic algorithms in continuous control. We theoretically show that actor networks cannot be effectively trained with transitions that have large TD errors. As a result, the approximate policy gradient computed under the Q-network diverges from the actual gradient computed under the optimal Q-function. Motivated by this, we introduce a novel experience replay sampling framework for actor-critic methods, which also regards issues with stability and recent findings behind the poor empirical performance of PER. The introduced algorithm suggests a new branch of improvements to PER and schedules effective and efficient training for both actor and critic networks. An extensive set of experiments verifies our theoretical claims and demonstrates that the introduced method significantly outperforms the competing approaches and obtains state-of-the-art results over the standard off-policy actor-critic algorithms. 
\end{abstract}

\keywords{deep reinforcement learning \and off-policy learning \and prioritized experience replay \and actor-critic algorithms}

\section{Introduction}
In off-policy deep reinforcement learning (RL), the experience replay buffer \citep{exp_replay}, which contains experiences that different policies may collect, is a vital ingredient of policy optimization \citep{jair_2}. Experience replay can stabilize and improve policy optimization by storing many previous experiences (or transitions) in a buffer and reusing them multiple times to perform gradient steps on policies and value functions approximated by deep neural networks \citep{sutton2018reinforcement}. Although the initial proposal of experience replay considered uniform sampling from the buffer, various sampling methods, e.g., \citep{per,ners,maper}, were shown to improve the data efficiency by calculating priority scores for the experiences. 

The use of priority-based non-uniform sampling in deep RL stems from a technique known as Prioritized Experience Replay (PER) \citep{per}, in which high error transitions are sampled with higher likelihood, allowing for faster learning and reward propagation by focusing on the most crucial data. In an ablation study, PER was shown to be the most key enhancement for the overall performance of the Deep Q-Network (DQN) algorithm \citep{dqn} compared to other improvements \citep{rainbow}. Although the motivation of PER is intuitive for learning in discrete action spaces, e.g., The Arcade Learning Environment \citep{jair_1}, many empirical studies showed that it substantially decreases the performance of RL agents in continuous action domains, resulting in suboptimal or random behavior \citep{lap, ners, maper}. Unfortunately, the poor performance of PER in continuous domains lacks a critical theoretical foundation. In this study, we develop an analysis that enables us to understand why PER cannot be effectively combined with actor-critic algorithms in continuous control and suggest novel modifications to PER to improve the empirical performance of the algorithm.

In continuous action spaces, the critic cannot be used to select actions due to the intractable, i.e., infinitely many possible actions \citep{sutton2018reinforcement}. Actor-critic methods can overcome this by utilizing a separate function, called \textit{actor}, to choose actions on the observed states. When combined with PER, the actor and critic networks are trained with transitions corresponding to large temporal-difference (TD) errors. TD error is a reasonable proxy that stands as the loss of the critic in Q-learning \citep{q_learning} and indicates the uncertainty and knowledge of the critic on the collected transitions in terms of the bootstrapped expected future rewards \citep{prioritized_sweeping}. Hence, in the basis of the TD-learning \citep{sutton88}, a large TD error implies that the critic has little knowledge and high uncertainty about the experiences \citep{sutton2018reinforcement}. However, we claim that actors cannot be effectively trained with experiences that the critic does not know their future returns well. An intuitive analogy may be that it is infeasible to expect a student to learn a subject well if the teacher has little knowledge about it. Our main theoretical contributions in this work justify our claim that if an actor-critic algorithm is trained with a transition corresponding to a large TD error, the approximate policy gradient, i.e., computed under the Q-network, can significantly diverge from the actual gradient, i.e., computed under the optimal Q-function, for the transition in interest or the subsequent transition. This finding can be used to improve the performance of PER by training the actor with different experiences and facilitating the design of novel prioritized
methods. Discoveries of this study are summarized as follows:
\begin{itemize}
    \item \textbf{Actor networks should be trained with low TD error transitions:} The critical implication of this finding is that the policy gradient, either stochastic or deterministic, that depends on the critic cannot be effectively computed using transitions on which the critic has high uncertainty. In particular, we find that a large TD error can correspond to a high Q-value estimation error for some transitions. Such error may cause the approximate policy gradient to diverge from the actual gradient under the optimal Q-function. To the best of our knowledge, this is the primary reason behind the poor performance of PER in standard off-policy actor-critic algorithms and we are the first to show it theoretically. Ultimately, this can only be overcome when the actor is trained with low TD error transitions in the TD error based prioritized sampling. 

    \item \textbf{Actor and critic networks should be optimized with uniformly sampled transitions for a fraction of the batch size:} Training actor and critic with completely different transitions, e.g., low and high TD error, violates the actor-critic theory \citep{konda} since critic parameters always depend on the actor parameters as the actions selected by the actor are used in the updates. Our empirical studies show that using a set of uniformly sampled transitions for a fraction of the batch size is extremely important in actor-critic training and ensures stability in learning.
    
    \item \textbf{Loss functions should be modified to prevent outlier bias leakage in the prioritized sampling:} While the combination of the latter two findings is theoretically and empirically favorable, prioritized and uniform sampling should not be considered in isolation from the loss function as outlier biased transitions may still leak during sampling \citep{lap}. Notably, we demonstrate that corrections to PER cannot reach their maximum potential unless the mean-squared error (MSE) in the Q-network training is corrected. Thus, we leverage the prominent results of \cite{lap} in our modifications to PER. 
\end{itemize}

In this paper, we introduce a novel experience replay prioritization framework, the Loss-Adjusted Approximate Actor Prioritized Experience Replay (LA3P) algorithm, to overcome the mentioned limitations of PER in continuous control. LA3P effectively adapts PER to actor-critic algorithms in continuous action spaces by training the actor network with transitions that the critic has reliable knowledge of. Moreover, our algorithm considers the issues with stability and traditional actor-critic theory by not completely separating the actor and critic networks and adjusting the loss function accordingly. We evaluate LA3P on challenging OpenAI Gym \citep{gym} continuous control benchmarks and find that our method outperforms the competing PER correction algorithms by a large margin and obtains significant gains over PER and state-of-the-art in all of the tasks tested. Furthermore, an extensive set of ablation studies verifies that each introduced modification to PER is essential to maintain the overall performance in actor-critic algorithms. All of our code and results are open-sourced and provided in the GitHub repository\footnote{\url{https://github.com/baturaysaglam/LA3P}\label{our_repo}}.

\section{Related Work}
Initial studies in experience prioritization consider prioritized sweeping for value iteration to boost the learning speed and effectively use the computational resources \citep{prioritized_sweeping,sweeping_nips}. \textit{Prioritized sweeping} is a model-based reinforcement learning approach that attempts to focus an agent's limited computing resources to get a reasonable analysis of the environment's state values. It is also utilized in modern applications of RL to perform importance sampling over the collected trajectories \citep{is_off_pol_pred} and learning from demonstrations \citep{deep_q_learning_from_demons}. 

Prioritized Experience Replay, which we investigate in depth in later sections, has been one of the most remarkable improvements to the DQN algorithm and its successors and is employed in many learning algorithms along with additional improvements such as Rainbow \citep{rainbow}, distributed PER \citep{distributed_per}, and distributional policy gradients \citep{distributed_pg}. Modifications on PER have also been proposed, e.g., prioritizing the sequences of transitions \citep{reactor} or optimization of the prioritization scheme \citep{ero}. A counterpart to the mentioned experience replay approaches is determining which transitions to favor or forget \citep{remember_and_forget_er}. Moreover, the effects originating from the composition and size of the experience replay buffers have also been studied by \cite{selected_er} and \cite{allerton}. The discussed methods usually consider model-based learning in discrete action domains, aiming at which sources to focus on or which features to select and store. Conversely, we focus on the problems that originate from the prioritized sampling in continuous control model-free deep RL to constitute an approach for deciding which experiences to replay or sample. 

Lately, the learning-based approaches through deep function approximators have been proposed to determine which experiences to sample, independent of the experience replay buffer composition and without deciding which transitions to store. \cite{ero} train a multi-layer perceptron whose input is the concatenation of reward, time step, and TD error. The network outputs a Bernoulli distribution to assign priorities to each sample within the replay buffer. On top of the neural sampler proposed by \cite{ero}, \cite{ners} also considers TD error, similar to PER. They train the sampler network through a novel REINFORCE \citep{reinforce} trick that measures the improvement in the evaluation rewards if the agent is trained with the priorities produced by the sampler network and updates it accordingly. In contrast to these learnable sampling strategies, we introduce corrections to the Prioritized Experience Replay in a rule-based manner. 

Recently, corrections to PER have been extensively investigated by \cite{lap} and \cite{maper}. First, \cite{lap} addressed that any loss function combined with non-uniform probability may be converted into a new uniformly sampled loss function counterpart with the same expected gradient. By using this finding, \cite{lap} completely replaces the loss in PER with such a modified loss function, which has been shown to have no effect on empirical performance. By correcting its uniformly sampled loss function equivalent, this connection provides a new branch of PER improvements called Loss Adjusted Prioritized (LAP) Experience Replay \citep{lap}. We broadly investigate LAP in the later sections. Secondly, \cite{maper} claim that sampling from the replay buffer depending highly on the TD error (or Q-network's error) may be ineffective due to the under- or overestimation of the Q-values resulting from the deep function approximators and bootstrapping. For this reason, they proposed to learn auxiliary features driven from the components in model-based RL to calculate the scores of experiences. The proposed method, Model-Augmented PER (MaPER) \citep{maper}, uses the critic network to improve the effect of curriculum learning for predicting Q-values with minimal memory and computational overhead compared to vanilla PER. Ultimately, we include the theoretical results of \cite{lap} in our methodology, while we also compare our method against uniform sampling, PER, LAP, and MaPER in our empirical studies.

\section{Technical Preliminaries}
\subsection{Deep Reinforcement Learning}
This study considers the standard RL setup, described by \cite{jair_3} and \cite{sutton2018reinforcement}, in which an agent interacts with an environment to collect rewards. At every discrete time step $t$, the agent observes a state $s$ and chooses an action $a$. Depending on its action selection, the agent receives a reward $r$ and observes the next state $s^{\prime}$. The RL paradigm is represented by a finite Markov Decision Process consisting of a 5-tuple $(\mathcal{S}, \mathcal{A}, \mathcal{R}, P, \gamma)$, with state space $\mathcal{S}$, action space $\mathcal{A}$, a deterministic or stochastic reward function $\mathcal{R}$, the environment dynamics model $P$, and the discount factor $\gamma \in [0, 1)$.

The performance of a policy $\pi$ is assessed under the action-value function (Q-function or critic) $Q^{\pi}$, which represents the expected sum of discounted rewards while following the policy $\pi$ after performing the action $a$ in state $s$:
\begin{equation}
    Q^{\pi}(s, a) = \mathbb{E}_{\pi}[\sum_{t = 0}^{\infty}\gamma^{t}r_{t + 1}|s_{0} = s, a_{0} = a].  
\end{equation}
The action-value function is determined through the Bellman equation \citep{belmann}:
\begin{equation}
    Q^{\pi}(s, a) = \mathbb{E}_{r, s^{\prime} \sim P, a^{\prime} \sim \pi}[r + \gamma Q^{\pi}(s^{\prime}, a^{\prime})],
\end{equation}
where $a^{\prime}$ is the next action selected by the policy on the observed next state $s^{\prime}$.

In deep RL, the critic is approximated by a deep neural network $Q_{\theta}$ with parameters $\theta$. Then, the Q-learning with deep function approximators transforms into the standard DQN algorithm. Given a transition tuple $\tau = (s, a, r, s^{\prime})$, DQN is trained by minimizing the loss $\delta_{\theta}(\tau)$ based on the temporal-difference error corresponding to the Q-network $Q_{\theta}$. The TD-learning algorithm is an update rule based on the Bellman equation, which defines a fundamental relationship used to learn the action-value function by bootstrapping from the value estimate of the current state-action pair $(s, a)$ to the subsequent state-action pair $(s^{\prime}, a^{\prime})$:
\begin{gather}
    y(\tau) = r + \gamma Q_{\theta^{\prime}}(s^{\prime}, a^{\prime});\label{eq:q_target} \\
    \delta_{\theta}(\tau) = y(\tau) - Q_{\theta}(\tau).
\end{gather}
Note that TD error in the latter equation can also be regarded as $\delta_{\theta}(\tau) = Q_{\theta}(\tau) - y(\tau)$, however, it does not affect our analysis. Transitions $\tau \in \mathcal{B}$ are sampled through a sampling method from the experience replay buffer that contains a previously collected set of experiences in the form of a batch of transitions $\mathcal{B}$. The target $y(\tau)$ in (\ref{eq:q_target}) utilizes a separate target network with parameters $\theta^{\prime}$ that maintains stability and fixed objective in learning the optimal Q-function. The target parameters are updated either with the soft or hard update by copying parameters $\theta$ from the Q-network $Q_{\theta}$ by a small fraction $\zeta$ at every step or to match $\theta$ every fixed period, respectively. In each update step, the loss for the Q-network is averaged over the sampled batch of transitions $\mathcal{B}$: $\frac{1}{\lvert\mathcal{B}\rvert}\sum_{\tau \in \mathcal{B}}\delta_{\theta}(\tau)$, where $\lvert\mathcal{B}\rvert$ is the number of transitions contained in $\mathcal{B}$. Note that the batch size does not affect our analysis on the prioritization as the expected gradient is not altered. 

In controlling continuous systems, i.e., continuous action domains, the maximum $\mathrm{max}_{\Tilde{a}}Q(s, \Tilde{a})$ to select actions is intractable due to an infinite number of possible actions. To overcome this, actor-critic methods employ a separate network to represent the policy $\pi_{\phi}$, parameterized by $\phi$, to determine the actions based on the observed states. The policy network is often regarded as the actor network, and it can be either deterministic $a = \pi_{\phi}(s)$ or stochastic $a \sim \pi_{\phi}(\cdot|s)$. The next action $a^{\prime}$ in the construction of the Q-network target in (\ref{eq:q_target}) can be chosen by the behavioral actor network $\pi_{\phi}$, i.e., the policy that the agent is using for action selection, or target actor network $\pi_{\phi^{\prime}}$, depending on the actor-critic method. Equivalently, the target actor network with parameters $\phi^{\prime}$ is used to ensure stability over the updates. Finally, the policy is optimized with respect to the policy gradient $\nabla{\phi}$ computed by a policy gradient technique, the loss of which is explicitly or implicitly based on maximizing the Q-value estimate of $Q_{\theta}$. Additionally, the Q-network is trained through the DQN algorithm or one of its variants, e.g., Clipped Double Q-learning \citep{td3}.

\subsection{Prioritized Experience Replay}
Prioritized Experience Replay is a non-uniform sampling strategy for replay buffers in which transitions are sampled in proportion to their TD error. The primary reasoning for PER is that training on the highest error samples will yield the most significant performance improvement. PER introduces two modifications over the standard uniform sampling. First, a stochastic prioritization scheme is used. The motivation is that TD errors are updated only for replayed transitions. As a result, the initially high TD error transitions are updated more frequently, resulting in a greedy prioritization. Also, the noisy Q-value estimates increase the variance due to the greedy sampling. Therefore, overfitting is inevitable if one directly samples transitions proportional to their TD errors. To remedy this, a probability value is assigned to each transition $\tau_{i}$, proportional to its TD error $\delta_{\theta}(\tau_{i})$, and set to the power of a hyper-parameter $\alpha$ to smooth out the extremes:
\begin{equation}
\label{eq:prioritization}
    p(\tau_{i}) = \frac{\lvert\delta_{\theta}(\tau_{i})\rvert^{\alpha} + \mu}{\sum_{j}(\lvert\delta_{\theta}(\tau_{j})\rvert^{\alpha} + \mu)},
\end{equation}
where a small constant $\mu$ is added to avoid assigning zero probabilities to transitions; otherwise, they would not be sampled again. This is required since the most recent value of a transition's TD error is approximated by the TD error when it was last sampled.

Second, favoring large TD error transitions with the stochastic prioritization shifts the distribution of $s^{\prime}$ to $\mathbb{E}_{s^{\prime}}[Q(s^{\prime}, a^{\prime})]$. This can be corrected through importance sampling with ratios $w(\tau_{i})$:
\begin{gather}
    \hat{w}(\tau_{i}) = \left(\frac{1}{\lvert R \rvert} \cdot \frac{1}{p(\tau_{i})}\right)^{\beta}, \\
    w(\tau_{i}) = \frac{\hat{w}(\tau_{i})}{\mathrm{max}_{j}\hat{w}(\tau_{j})}; \\
    \mathcal{L}_{\mathrm{PER}}(\delta_{\theta}(\tau_{i})) = w(\tau_{i})\mathcal{L}(\delta_{\theta}(\tau_{i}))\label{eq:per_is},  
\end{gather}
where $\mathcal{L}$ indicates the loss, $\lvert R \rvert$ is the total number of transitions in the replay buffer, and the hyper-parameter $\beta$ is used to smooth out the high variance induced by the importance sampling weights. With the latter equation, the distribution shift is corrected such that the effect of high priorities is reduced by using a ratio between uniform sampling with probability $\frac{1}{\lvert R \rvert}$ and the ratio in (\ref{eq:prioritization}). Lastly, the $\beta$ value is annealed from a pre-defined initial value $\beta_{0}$ to 1, to eliminate the bias introduced by the distributional shift.  

\section{Prioritized Sampling in Actor-Critic Algorithms}
We start by building the theoretical foundations for the performance degradation of the prioritized sampling in actor-critic methods. First, we demonstrate that there exist transitions such that the associated TD errors are directly proportional to the Q-value estimation errors by which the policy gradient is computed. Then, using our theoretical implications, we show that the gradient of the approximate actor network computed under the Q-network diverges from the actual gradient computed under the optimal Q-function if the policy is optimized using transitions with large TD error. In addition to our theoretical investigation, we address the existing problems shown by a prior study that explains the poor performance of PER with standard off-policy continuous control methods. 
\begin{lemma}
\label{lem:prop}
    If $\delta_{\theta}$ is the temporal-difference error associated with the critic network $Q_{\theta}$, then there exists a transition tuple $\tau_{t} = (s_{t}, a_{t}, r_{t}, s_{t + 1})$ with $\delta_{\theta}(\tau_{t}) \neq 0$ such that the absolute temporal-difference error on $\tau_{t}$ is directly proportional to the absolute estimation error on at least $\tau_{t}$ or $\tau_{t + 1}$:
    \begin{equation}
       \lvert \delta_{\theta}(\tau_{t}) \rvert \propto \lvert Q_{\theta}(s_{i}, a_{i}) - Q^{\pi}(s_{i}, a_{i}) \rvert; \quad i = t \vee (t + 1),
    \end{equation}
    where $Q^{\pi}(s_{i}, a_{i})$ is the actual Q-value of the state-action pair $(s_{i}, a_{i})$ while following the policy $\pi$. 
\end{lemma}
\begin{proof}
    The proof does not consider the target networks since they aim to ensure stability and fixed objective over the updates and have no effect on the estimation \citep{dqn}. First, expand $\delta_{\theta}(\tau_{t})$ to the bootstrapped value estimation form:
    \begin{equation}
    \label{eq:td_error}
        \delta_{\theta}(\tau_{t}) = r_{t} + \gamma Q_{\theta}(s_{t + 1}, a_{t + 1}) - Q_{\theta}(s_{t}, a_{t}) ,
    \end{equation}
    where $a_{t + 1} \sim \pi_{\phi}(s_{t + 1})$ is the action selected by the policy network $\pi_{\phi}$ on the observed next state $s_{t + 1}$. We know that the optimal action-value function $Q^{\pi}$ under the policy $\pi$ yields no TD error:
    \begin{equation}
    \label{eq:zero_td_error}
        \delta^{\pi}(\tau_{t}) = r_{t} + \gamma Q^{\pi}(s_{t + 1}, a_{t + 1}) - Q^{\pi}(s_{t}, a_{t}) = 0.
    \end{equation}
    Then, subtracting (\ref{eq:zero_td_error}) from (\ref{eq:td_error}) yields:
    \begin{equation}
       \delta_{\theta}(\tau_{t}) = \underbrace{(Q^{\pi}(s_{t}, a_{t}) - Q_{\theta}(s_{t}, a_{t}))}_{ \coloneqq x} + \gamma \underbrace{\left(Q_{\theta}(s_{t + 1}, a_{t + 1}) - Q^{\pi}(s_{t + 1}, a_{t + 1})\right)}_{\coloneqq y} \neq 0. 
    \end{equation}
    It is clear that $x$ is the estimation error at time step $t$ and $y$ is the estimation error at the subsequent time step $t + 1$. Note that the existence of an estimation error does not depend on the sign in the latter equation. For simplicity, we express the TD error in terms of $x$ and $y$:
    \begin{equation}
        \delta_{\theta}(\tau_{t}) = x + \gamma y \neq 0.
    \end{equation}
    Clearly, if $x = 0$, then $y \neq 0$, or vice versa, since $\gamma \geq 0$. Thus, there exists an estimation error by $Q_{\theta}$ either on $\tau_{t}$ or $\tau_{t + 1}$ if the TD error corresponding to the Q-network is non-zero. Notice that the absolute value of the TD error in the latter equation can be directly proportional to $\lvert x \rvert$ or $\lvert y \rvert$. For instance, suppose that $\delta_{\theta}(\tau_{t}) < 0$, $x \geq 0$, and $y < 0$. In such a case, an increasing TD error increases the absolute estimation error of $\lvert y \rvert$ if $x$ remains constant. One can notice that this can be the case for each combination of the signs of $x$, $y$, and $\delta_{\theta}(\tau_{t})$, depending on the transition tuple $\tau_{t}$. Hence, we infer that there may exist a direct proportionality between the absolute temporal-difference error $\lvert \delta_{\theta}(\tau_{t}) \rvert$ at time step $t$ and the absolute estimation error for $\tau_{t}$ or $\tau_{t + 1}$.
\end{proof}

The estimation error in RL with function approximation is often caused by using function approximators and bootstrapping, such as in the Q-learning algorithm \citep{sutton2018reinforcement}. As the estimation error is not usually distinguished into the function approximation and bootstrapping, it cannot be deduced that there is \textit{always} a direct proportionality between estimation error and TD error. In addition, the TD-learning can be a poor estimate in some conditions, such as when the rewards are noisy \citep{sutton2018reinforcement}. The TD error is a measure of how unexpected or surprising a transition is, not the estimation accuracy \citep{prioritized_sweeping}. Therefore, Lemma \ref{lem:prop} has to be made to show that this correlation exists for \textit{some} state-action pairs. Next, we leverage our latter result to explain why policy optimization cannot be effectively performed using transitions with large TD errors. 

\begin{theorem}
\label{thm:diverge_grad}
Let $\tau_{i}$ be a transition such that Lemma \ref{lem:prop} is satisfied. Then, if $\delta_{\theta}(\tau_{i}) \neq 0$, the following relation holds:
\begin{equation}
    \lvert \delta_{\theta}(\tau_{i}) \rvert \propto \lvert \nabla\phi(\tau_{j}) - \nabla\phi_{\text{true}}(\tau_{j}) \rvert;  \quad j = i \vee (i + 1), 
\end{equation}
where $\nabla\phi(\tau_{j})$ and $\nabla\phi_{\text{true}}(\tau_{j})$ are the resulting policy gradients corresponding to $\tau_{j}$ if computed under the Q-network $Q_{\theta}$ and optimal Q-function $Q^{\pi}$, respectively. 
\end{theorem}
\begin{proof}
    The proof follows from Lemma \ref{lem:prop} and adaptation of the policy gradient theorem of \cite{sutton_policy_gradient} to the deep function approximation. To begin the proof, we first formally express the standard policy iteration with function approximation in terms of the policy parameters $\phi$:
    \begin{equation}
        \phi \leftarrow \phi + \eta\nabla\phi(s_{i}, a_{i}),
    \end{equation}
    where $\nabla\phi(s_{t}, a_{i})$ is the policy gradient computed for the state-action pair $(s_{i}, a_{i}) \in \tau_{i}$ and $\eta$ is the learning rate. \cite{sutton_policy_gradient} provides a general formulation for the policy gradient in policy iteration with function approximation as:
    \begin{equation}
    \label{ccc}
        \nabla\phi(s_{i}, a_{i}) \coloneqq \sum_{s}d^{\pi}(s_{i})\sum_{a}\frac{\partial\pi(s_{i}, a_{i})}{\partial\phi}f_{\theta}(s_{i}, a_{i}),
    \end{equation}
    where $f_{\theta}$ is the function approximation to $Q^{\pi}$, i.e., $f_{\theta} \coloneqq Q_{\theta}$, and $d^{\pi}(s_{i})$ is a discounted weighting of encountered states starting at $s_{0}$ and then following $\pi$, denoted by:
    \begin{equation}
        d^{\pi}(s_{i}) = \sum_{t = 0}^{\infty}\gamma^{t}p(s_{t} = s_{i} | s_{0}, \pi).
    \end{equation}
    Clearly, the gradient of the policy parameters $\phi$ is proportional to the gradient of $\pi(s_{i}, a_{i})$ with respect to $\phi$ weighted by the estimated Q-value of $(s_{i}, a_{i})$. Here, we can neglect $d^{\pi}(s_{i})$ since policy parameters $\phi$ has not effect on $d^{\pi}(s_{i})$ \citep{sutton_policy_gradient}. Then, using (\ref{ccc}), \cite{sutton_policy_gradient} showed that:
    \begin{equation}
    \label{aaa}
        \nabla\phi(\tau_{i}) \propto \frac{\partial\pi_{\phi}(s_{i}, a_{i})}{\partial\phi}Q_{\theta}(s_{i}, a_{i}),
    \end{equation}
    where we denote that $\nabla\phi(\tau_{i})$ is the policy gradient computed with respect to the imperfect Q-value estimate $Q_{\theta}(s_{i}, a_{i})$. If the policy is stochastic, the latter equation transforms into:
    \begin{equation}
    \label{bbb}
        \nabla\phi(\tau_{i}) \propto \frac{\partial\log\pi_{\phi}(a_{i}|s_{i})}{\partial\phi}Q_{\theta}(s_{i}, a_{i}).
    \end{equation}
    Notice that (\ref{aaa}) and (\ref{bbb}) correspond deterministic and stochastic policies, respectively. Therefore, we generalize that the gradients of deterministic and stochastic policies are proportional to the Q-value estimates of the Q-network. Since the Q-network approximates the actual Q-function $Q^{\pi}$, an error exists in the value estimates of $Q_{\theta}$. We express the policy gradient in terms of the true Q-value and estimation error, which is valid for both deterministic and stochastic policies: 
    \begin{equation}
    \label{ffff}
        \nabla\phi(\tau_{i}) \propto (Q^{\pi}(s_{i}, a_{i}) + \epsilon_{\tau_{i}}),
    \end{equation}
    where we neglect the derivatives as they do not depend on the critic and $\epsilon_{\tau_{i}} \in \mathbb{R}$ is the error caused by bootstrapping and function approximation on the estimation of $Q^{\pi}(s_{i}, a_{i})$, i.e., $Q_{\theta}(s_{i}, a_{i}) = Q^{\pi}(s_{i}, a_{i}) + \epsilon_{\tau_{i}}$. Naturally, the actual Q-value of $(s_{i}, a_{i})$ computed under the true Q-function $Q^{\pi}$ associates with the actual policy gradient $\nabla\phi_{\text{true}}(\tau_{i})$:
    \begin{equation}
    \label{eq:true_grad_prop}
        \nabla\phi_{\text{true}}(\tau_{i}) \propto Q^{\pi}(s_{i}, a_{i}).
    \end{equation}
    From (\ref{ffff}) and (\ref{eq:true_grad_prop}), we observe that an increasing absolute estimation error increases the divergence from the actual policy gradient since the Q-value estimate $Q_{\theta}(s_{i}, a_{i})$ moves away from the actual Q-value $Q^{\pi}(s_{i}, a_{i})$, which alters the weights used in policy gradient computation, i.e., (\ref{aaa}) and (\ref{bbb}). We formally express this result as:
    \begin{equation}
    \label{eeee}
        \lvert \epsilon_{\tau_{i}} \rvert \propto \lvert \nabla\phi(\tau_{i}) - \nabla\phi_{\text{true}}(\tau_{i}) \rvert.
    \end{equation}
    According to Lemma \ref{lem:prop}, we know that a large absolute TD error can correspond to a large absolute estimation error for $\tau_{i}$ or $\tau_{i + 1}$. Hence, given (\ref{eeee}), the absolute estimation error in the current or subsequent step is directly proportional to the absolute TD error in the current step:
    \begin{equation}
    \label{eq:error_prop}
        \lvert \delta_{\theta}(\tau_{i}) \vert \propto \lvert \epsilon_{\tau_{j}} \rvert, \quad j = i \vee (i + 1).
    \end{equation}
    From the latter equation, we deduce that an increasing absolute TD error in the current step increases the divergence from the actual policy gradient in the current or subsequent step:
    \begin{equation}
        \lvert \delta_{\theta}(\tau_{i}) \rvert \propto \lvert \nabla\phi(\tau_{j}) - \nabla\phi_{\text{true}}(\tau_{j}) \rvert; \quad j = i \vee (i + 1).
    \end{equation}
\end{proof}

Theorem \ref{thm:diverge_grad} states that if the actor network is optimized with a transition corresponding to a large TD error, the resulting policy gradient at the current or subsequent step may diverge from the actual gradient. We formally state this finding in Corollary \ref{cor:diverge_grad}.
\begin{corollary}
\label{cor:diverge_grad}
    If the temporal-difference error of a transition increases, the approximate policy gradient computed by any policy gradient algorithm with respect to the Q-network can diverge from the actual gradient computed under the optimal Q-function for the current or subsequent transition.
\end{corollary}

This forms an essential ingredient in the degraded performance of the prioritized sampling when an actor network is employed. We now address a recent finding that explains a complement to the poor performance of the prioritized sampling in continuous control. As discussed, prioritization by transitions with large TD error through stochastic sampling shifts the distribution of $s^{\prime}$ to $\mathbb{E}_{s^{\prime}}[Q(s^{\prime}, a^{\prime})]$. Therefore, this induced bias is corrected by importance sampling expressed in (\ref{eq:per_is}). However, \cite{lap} argued that PER does not entirely eliminate the bias and may favor outliers when combined with MSE loss in the Q-network updates. Particularly, the implication of the use of MSE combined with PER is that when PER is optimized with respect to a loss $\frac{1}{\rho}\lvert\delta_{\theta}(\tau)\rvert^{\rho}$ and such that $\rho + \alpha - \alpha\beta \neq 2$, the Q-network target and hence, the Q-network updates become biased. This result is highlighted in Remark \ref{rem:tau_bias} and claimed by \cite{lap} to be one of the reasons of the poor performance of PER with standard actor-critic algorithms in continuous action domains. 
\begin{remark}
\label{rem:tau_bias}
    The PER objective is biased if $\rho + \alpha - \alpha\beta \neq 2$ under the loss function $\frac{1}{\rho}\lvert\delta_{\theta}(\tau)\rvert^{\rho}$ \citep{lap}. 
\end{remark}

Given Remark \ref{rem:tau_bias}, it is inferred by \cite{lap} that prioritized sampling is biased when combined with MSE since $\rho + \alpha - \alpha\beta \neq 2$ is never exactly satisfied, i.e., if $\beta < 1$, then $2 + \alpha - \alpha\beta > 2$ for $\alpha \in (0, 1]$. Moreover, the induced bias is not always the same. On the contrary, an L1 loss can satisfy such property such that $1 + \alpha - \alpha\beta \in [1, 2]$ for $\alpha \in (0, 1]$ and $\beta \in [0, 1]$ since $\rho = 1$ in the L1 loss. In practice, however, the L1 loss may not be ideal since each update entails a constant-sized step, which may cause the objective to be overstepped if the learning rate is too large. To overcome this, \cite{lap} applied the commonly used Huber loss with $\kappa = 1$:
\begin{gather}
\label{eq:huber_loss}
    \mathcal{L}_{\mathrm{Huber}}(\delta_{\theta}(\tau_{i})) = \left\{
        \begin{array}{ll}
            0.5\delta_{\theta}(\tau_{i})^{2} & \quad \text{if $\lvert\delta_{\theta}(\tau_{i})\rvert \leq \kappa$}, \\
            \lvert\delta_{\theta}(\tau_{i})\rvert & \quad \text{otherwise}.
        \end{array}
        \right.
\end{gather}
When the error is below threshold one, the Huber loss swaps from L1 to MSE and properly scales the gradient as $\delta_{\theta}(\tau_{i})$ approaches zero. When $\lvert\delta_{\theta}(\tau_{i})\rvert < 1$, MSE is applied. Thus, samples with an error of less than one should be sampled uniformly to prevent the bias induced by MSE and prioritization. This is achieved in the LAP algorithm by the prioritization scheme $p(\tau_{i}) = \mathrm{max}(\lvert\delta_{\theta}(\tau_{i})\rvert^{\alpha}, 1)$, through which samples with low priority are clipped to be at least one. Finally, the LAP algorithm overcomes the problem noted in Remark \ref{rem:tau_bias} by combining the Huber loss expressed in (\ref{eq:huber_loss}) and the following modified stochastic prioritization scheme:
\begin{equation}
\label{tttt}
    p(\tau_{i}) = \frac{\mathrm{max}(\lvert\delta_{\theta}(\tau_{i})\rvert^{\alpha}, 1)}{\sum_{j}\mathrm{max}(\lvert\delta_{\theta}(\tau_{j})\rvert^{\alpha}, 1)}.
\end{equation}

The results presented by \cite{lap} form a complement to our theoretical conclusions for the poor performance of PER, which we underline in Remark \ref{rem:fujimoto_res}. We note that our emphasis is on the continuous control algorithms in which actions are selected by a separate actor network trained to maximize the action-value estimate of the Q-network. In contrast, Remark \ref{rem:fujimoto_res} comprises both discrete and continuous off-policy deep RL. There may be another justification for the poor performance of PER, such as inaccurate Q-value estimates. Nevertheless, such reasons are not PER-dependent and are induced by the DQN variants. Therefore, to the best of our knowledge, we believe that Corollary \ref{cor:diverge_grad} and Remark \ref{rem:fujimoto_res} establish the basis for the algorithmic drawbacks of PER in continuous control. 
\begin{remark}
\label{rem:fujimoto_res}
     To further eliminate the bias favoring the outlier transitions in the Prioritized Experience Replay algorithm, the Huber loss with $\kappa = 1$ should be used in combination with the prioritization scheme expressed in (\ref{tttt}) \citep{lap}.
\end{remark}

\section{Adaptation of Prioritized Experience Replay to Actor-Critic Algorithms}
\subsection{Inverse Sampling for the Actor Network}
To remedy the mentioned issues of prioritized sampling in actor-critic algorithms, we introduce a set of novel modifications to vanilla PER. We start with Corollary \ref{cor:diverge_grad}, that is, if the actor network is trained with a large TD error transition, the approximate policy gradient diverges from the actual gradient, at least for the transition in interest or the subsequent transition. However, if the policy gradient diverges for the subsequent transition and the subsequent transition does not correspond to a large TD error, the performance might not degrade under the TD error based prioritization scheme. Nonetheless, this is a slight possibility since the replay buffer contains few transitions in the initial optimization steps, the number of which is smaller than the batch size. 
\begin{observation}
    The performance of actor-critic methods might not degrade under the PER algorithm if the transitions subsequent to the sampled transitions are not used to optimize the actor network. Nevertheless, this remains a slight possibility in the standard off-policy algorithms in continuous control.
\end{observation}

We overcome the issue expressed in Corollary \ref{cor:diverge_grad} by optimizing the actor network with transitions that have small TD errors. Therefore, the requirement in Corollary \ref{cor:diverge_grad} can be achieved by inverse sampling from the prioritized replay buffer, i.e., sampling low TD error transitions for the actor network through the PER approach. For the implementation, the PER data structure should be analyzed. An efficient and the most popular implementation of PER, which also corresponds to the \textit{proportional prioritization}, is based on a ``sum-tree'' data structure. In principle, the sum-tree data structure used in PER is extremely similar to the binary heap's array representation. Instead of the standard heap property, a parent node's value is equal to the sum of its children. Internal nodes are intermediate sums, with the parent node carrying the sum over all priorities $p_{\text{total}}$. Leaf nodes store transition priorities, and internal nodes are intermediate sums. This allows for $\mathcal{O}(\log \lvert R \rvert)$ updates and sampling while calculating the cumulative total of priorities. The range $[0, p_{\text{total}}]$ is equally split into $N$ ranges to sample a mini-batch of size $N$. Then, each range is uniformly sampled for a value. Finally, the tree is queried for the transitions corresponding to each sampled value. \newline

An instinctive approach to sample transitions with a probability inversely proportional to the TD error accommodates creating a new sum tree containing the priorities' global inverse. Although priorities in vanilla PER are updated for every training step through a previously defined sum tree data structure, inverse sampling for actor updates requires creating a new sum tree prior to training. In every update step, priorities are calculated as the following:
\begin{equation}
    \label{inverse_eq_1}
    I \sim \Tilde{p}(\tau_{i}) = \frac{p_{\text{max}}}{p(\tau_{i})} =  \mathrm{max}_{i}\left(\frac{\mathrm{max}(\lvert\delta_{\theta}(\tau_{i})\rvert^{\alpha}, 1)}{\sum_{j}\mathrm{max}(\lvert\delta_{\theta}(\tau_{j})\rvert^{\alpha}, 1)}\right) \cdot \frac{\sum_{j}\mathrm{max}(\lvert\delta_{\theta}(\tau_{j})\rvert^{\alpha}, 1)}{\mathrm{max}(\lvert\delta_{\theta}(\tau_{i})\rvert^{\alpha}, 1)},
\end{equation}
where $p(\tau_{i})$ is the priority of the $i^{\text{th}}$ transition and $p_\mathrm{max}$ is the maximum of the stored transitions' previously determined priorities. As discussed in Remark \ref{rem:fujimoto_res}, the use of MSE with PER still induces varying biases that may favor outlier transitions. Hence, we adopt the prioritization scheme of LAP, expressed by (\ref{tttt}), in (\ref{inverse_eq_1}). Notice that (\ref{inverse_eq_1}) does not alter proportional prioritization. The relative proportions, e.g., the largest over the smallest, do not change as we take the inverse by multiplication.

This forms the fundamental component of our approach. To eliminate such bias in the Q-network counterpart, we employ the Huber loss with $\kappa = 1$ defined in (\ref{eq:huber_loss}) for the Q-network updates, the same in the LAP algorithm. Therefore, at every optimization step $t$, Q-network and priorities are updated, respectively, as:
\begin{gather}
    I \sim p(\tau_{i}) = \frac{\mathrm{max}(\lvert\delta_{\theta}(\tau_{i})\rvert^{\alpha}, 1)}{\sum_{j}\mathrm{max}(\lvert\delta_{\theta}(\tau_{j})\rvert^{\alpha}, 1)}; \\
    \theta \leftarrow \theta - \eta \cdot \frac{1}{\lvert I \rvert}\sum_{i \in I}\nabla_{\theta}\mathcal{L}_{\mathrm{Huber}}(\delta_{\theta}(\tau_{i})), \\
    p(\tau_{i}) \leftarrow \mathrm{max}(\lvert\delta_{\theta}(\tau_{i})\rvert^{\alpha}, 1)\text{ for $i \in I$},\label{eq:priority_upt}
\end{gather}
where $I$ are indices of the sampled batch of prioritized transitions. Note that the clipping reduces the likelihood of dead transitions when $p(\tau_{i}) \approx 0$, which eliminates the need for the $\mu$ parameter. We also note that the Huber loss cannot apply to the computation of the policy gradient. There are several reasons for this. First, the priorities are determined by the Q-network's loss, i.e., TD error, and the use of MSE loss in combination with TD error based prioritized sampling is the main cause for the mentioned outlier bias. Moreover, the policy loss and gradient are calculated by a class of policy gradient techniques that cannot be replaced. Hence, the outlier bias does not affect the policy gradient, and there is no need to employ Huber loss for the policy network.

\subsection{Optimizing the Actor and Critic with a Shared Set of Transitions}
As it may occur in our case, optimizing the actor and critic networks with entirely different transitions can violate the actor-critic theory. In general, features used by the critic network depend on the actor parameters and policy gradient since the actor determines the actions, and these actions lead to the observed state space \citep{konda}. A trivial corollary is that if the critic is updated by a set of features that lie in a state-action space in which the actor is never optimized, a substantial instability may occur since the transitions used by the critic are processed through the actor and the actor never sees those transitions \citep{konda}. Therefore, the reliability of critic's action evaluations might become questionable.

Intuitively, this situation can occur when inverse prioritized and prioritized sampling are used for the actor and critic networks, respectively, since they may never be optimized with the same transitions. This can be the case when the TD error of the critic's samples are not decreased such that the actor never sees them. We indicated that there is not always a direct correlation between TD and estimation errors. Hence, some of the transitions may initially have low TD errors. If the actor is optimized with respect to these low TD error transitions throughout the learning and the Q-network only focuses on the remaining large TD error transitions, samples used in the actor and critic training might not be the same. Although this remains a little possibility, we nevertheless overcome this by updating the actor and critic networks through a set of shared transitions, being a fraction of the sampled mini-batch of transitions. However, we cannot know the value of such a fraction, and we will introduce it as a hyper-parameter later. We emphasize these deductions in Observation \ref{cor:uniform} and regulate our approach accordingly. 
\begin{observation}
\label{cor:uniform}
    If the transitions for the actor and critic are sampled through inverse prioritized and prioritized sampling, respectively, they may never observe the same transitions, which violates the actor-critic theory \citep{konda} and induces instability in the learning. Hence, actor and critic should be optimized with the same set of transitions, at least for a fraction of the sampled mini-batch of transitions in every update step.
\end{observation}

\paragraph{How to choose the set of shared transitions?} We inspect the following sampling alternatives in terms of the TD error. We can additionally investigate auxiliary variables driven by the components in model-based RL to compute the scores of the experiences, similar to the work by \cite{maper}. However, learning additional features introduces additional computational overhead, which is not in this study's scope since we only focus on prioritization regarding the TD error and corrections to vanilla PER in continuous action spaces.
\begin{itemize}
    \item \textbf{Transitions with large TD error:} The actor network cannot be optimized with experiences that have large TD error since the policy gradient significantly diverges from the actual gradient, as discussed in Corollary \ref{cor:diverge_grad}.
    \item \textbf{Transitions with small TD error:} Learning from the experiences with small TD error can be beneficial, yet they decrease the sampling efficiency and waste resources since the Q-network has little to learn from small TD error transitions in the sense of prioritized sampling.   
    \item \textbf{Uniformly sampled transitions:} The latter two alternatives imply that uniform sampling for the set of shared transitions remains the only choice. Although large TD error transitions might be included in the uniformly sampled mini-batch, their effects are reduced due to averaging in the mini-batch learning.
\end{itemize}

From our latter discussion, we infer that uniform sampling for a fraction of the transitions in the sampled mini-batch is the optimal and only solution to the issue highlighted in Observation \ref{cor:uniform}. We may also utilize transitions with an average magnitude of TD errors, yet, uniform sampling already corresponds to the transitions with mean TD error in the expectation. Moreover, relying on random sampling can include transitions that cannot be sampled by prioritized and inverse prioritized sampling.

As discussed in Remark \ref{rem:fujimoto_res}, the combination of Huber loss ($\kappa = 1$) with the prioritized sampling can eliminate outlier bias in the LAP algorithm. \cite{lap} also introduced the mirrored loss function of LAP, with an equivalent expected gradient, for uniform sampling from the experience replay buffer. To observe the same benefits of LAP also in the uniform sampling counterpart, its mirrored loss function, Prioritized Approximate Loss (PAL), should be employed instead of MSE. Similar to the case of the LAP function, the PAL loss is also not employed in the policy network's updates. The PAL function is expressed as:
\begin{align}
    \label{eq:pal}
    \xi &= \frac{\sum_{j}\mathrm{max}(\lvert\delta_{\theta}(\tau_{j})\rvert^{\alpha}, 1)}{N}; \\
    \mathcal{L}_{\mathrm{PAL}}(\delta_{\theta}(\tau_{i})) &= \frac{1}{\xi}\left\{
        \begin{array}{ll}
            0.5\delta_{\theta}(\tau_{i})^{2} & \quad \text{if $\lvert\delta_{\theta}(\tau_{i})\rvert \leq 1$}, \\
            \frac{\lvert\delta_{\theta}(\tau_{i})\rvert^{1 + \alpha}}{1 + \alpha} & \quad \text{otherwise}.
        \end{array}
        \right.
\end{align}

\begin{remark}
\label{rem:pal}
   To eliminate the outlier bias in the uniform sampling counterpart, the Prioritized Approximate Loss (PAL) function should be employed, which has the same expected gradient as the Huber loss when combined with PER \citep{lap}.
\end{remark}

\begin{algorithm*}[htbp]
    \caption{Actor-Critic with Loss-Adjusted Approximate Actor Prioritized Experience Replay (LA3P)}
    \begin{algorithmic}[1]
    \STATE \textbf{Input:} Mini-batch size $N$, exponents $\alpha$ and $\beta$, uniform sampling fraction $\lambda$, target step-size $\zeta$, and actor and critic step-sizes $\eta_{\pi}$ and $\eta_{Q}$
    \STATE Initialize actor $\pi_{\phi}$ and critic $Q_{\theta}$ networks, with random parameters $\phi$ and $\theta$
    \STATE Initialize target networks $\phi^{\prime} \leftarrow \phi$, $\theta^{\prime} \leftarrow \theta$, if required
    \STATE Initialize $p_{\text{init}} = 1$ and the experience replay buffer $R = \emptyset$
    \FOR{$t = 1$ \textbf{to} $T$}
        \STATE Select action $a_{t}$ and observe reward $r_{t}$ and new state $s_{t + 1}$
        \STATE Store the transition tuple $\tau_{t} = (s_{t}, a_{t}, r_{t}, s_{t + 1})$ in $R$ with initial priority $p_{t} = p_{\text{init}}$
        \FOR{each update step}
            \STATE Uniformly sample a mini-batch of transitions: $I \sim p(\tau_{i}) = \frac{1}{\lvert R \rvert}; \quad \lvert I \rvert = \lambda \cdot N$
            \STATE Optimize the critic network: $\theta \leftarrow \theta - \eta_{Q} \cdot \frac{1}{\lvert I \rvert}\sum_{i \in I}\nabla_{\theta}\mathcal{L}_{\mathrm{PAL}}(\delta_{\theta}(\tau_{i}))$
            \STATE Compute the policy gradient $\nabla\phi(\tau_{i})$ for $\tau_{i}$ where $i \in I$
            \STATE Optimize the actor network: $\phi \leftarrow \phi + \eta_{\pi} \cdot \frac{1}{\lvert I \rvert}\sum_{i \in I}\nabla\phi(\tau_{i})$
            \STATE Update the priorities of the uniformly sampled transitions: $p(\tau_{i}) \leftarrow \mathrm{max}(\lvert\delta_{\theta}(\tau_{i})\rvert^{\alpha}, 1)$ for $i \in I$
            \STATE Update target networks if required: $\theta^{\prime} \leftarrow \zeta\theta + (1 - \zeta)\theta^{\prime}$, $\phi^{\prime} \leftarrow \zeta\phi + (1 - \zeta)\phi^{\prime}$
            \STATE Sample a mini-batch of transitions through prioritized sampling: \\ $I \sim p(\tau_{i}) = \frac{\mathrm{max}(\lvert\delta_{\theta}(\tau_{i})\rvert^{\alpha}, 1)}{\sum_{j}\mathrm{max}(\lvert\delta_{\theta}(\tau_{j})\rvert^{\alpha}, 1)}; \quad \lvert I \rvert = (1 - \lambda) \cdot N$
            \STATE Optimize the critic network: $\theta \leftarrow \theta - \eta_{Q} \cdot \frac{1}{\lvert I \rvert}\sum_{i \in I}\nabla_{\theta}\mathcal{L}_{\mathrm{Huber}}(\delta_{\theta}(\tau_{i}))$
            \STATE Update the priorities of the prioritized transitions: $p(\tau_{i}) \leftarrow \mathrm{max}(\lvert\delta_{\theta}(\tau_{i})\rvert^{\alpha}, 1)$ for $i \in I$
            \STATE Sample a mini-batch of transitions through inverse prioritized sampling: \\ $I \sim \Tilde{p}(\tau_{i}) = \frac{p_{\text{max}}}{p(\tau_{i})} =  \mathrm{max}_{i}\left(\frac{\mathrm{max}(\lvert\delta_{\theta}(\tau_{i})\rvert^{\alpha}, 1)}{\sum_{j}\mathrm{max}(\lvert\delta_{\theta}(\tau_{j})\rvert^{\alpha}, 1)}\right) \cdot \frac{\sum_{j}\mathrm{max}(\lvert\delta_{\theta}(\tau_{j})\rvert^{\alpha}, 1)}{\mathrm{max}(\lvert\delta_{\theta}(\tau_{i})\rvert^{\alpha}, 1)}$
            \STATE Compute the policy gradient $\nabla\phi(\tau_{i})$ for $\tau_{i}$ where $i \in I$
            \STATE Optimize the actor network: $\phi \leftarrow \phi + \eta_{\pi} \cdot \frac{1}{\lvert I \rvert}\sum_{i \in I}\nabla\phi(\tau_{i})$
            \STATE Update target networks if required: $\theta^{\prime} \leftarrow \zeta\theta + (1 - \zeta)\theta^{\prime}$, $\phi^{\prime} \leftarrow \zeta\phi + (1 - \zeta)\phi^{\prime}$
        \ENDFOR
    \ENDFOR
    \end{algorithmic}
    \label{alg:la3p}
\end{algorithm*}

Having the latter component included, this forms our PER correction algorithm, \textbf{L}oss-\textbf{A}djusted \textbf{A}pproximate \textbf{A}ctor \textbf{P}rioritized Experience Replay (LA3P). To summarize our approach, in each update step, a mini-batch of transitions of size $\lambda \cdot N$ is uniformly sampled, where $\lambda \in [0, 1]$ denotes the fraction of the uniformly sampled transitions, being the only hyper-parameter introduced by our approach. With the uniformly sampled batch, critic and actor networks are optimized, respectively. The critic is updated with respect to the PAL function expressed in (\ref{eq:pal}), and the priorities are updated after that. Next, $(1 - \lambda) \cdot N$ number of transitions are sampled for the critic and actor networks through prioritized and inverse prioritized sampling, respectively. Then, the critic is optimized by Huber loss ($\kappa = 1$) expressed in (\ref{eq:huber_loss}), and a policy gradient technique optimizes the actor. Finally, the priorities are again updated. In total, actor and critic networks are optimized by $N$ transitions per update step, as in standard off-policy actor-critic algorithms. Note that we update the priorities also in the uniform counterpart since the score of the transitions should be up-to-date whenever possible. In addition, the order of the prioritized and uniform updates does not alter the expected gradient. Hence, it does not matter which of them is used first. Nevertheless, in our implementation and experiments, we employ the structure outlined in Algorithm \ref{alg:la3p}, denoted through a generic application to off-policy actor-critic methods. A clear summary of the LA3P framework is also depicted in Appendix \ref{fig:simple_la3p}. 

Lastly, we perform a complexity analysis for our approach. The LA3P framework introduces an additional sum tree, the LAP, and PAL functions on top of vanilla PER. As LAP and PAL operate on the sampled batches of transitions, the computational complexity introduced by these modifications is dominated by the additive sum tree, which operates on the entire replay buffer. Moreover, the additive sum tree requires a priority update. Setting the priorities of the nodes has the same complexity as in PER. Additionally, LA3P takes the inverse of the priorities by multiplication, which takes $\mathcal{O}(\lvert R \vert)$ run time. Therefore, LA3P operates in $\mathcal{O}(\log \lvert R \rvert) + \mathcal{O}(\lvert R \rvert)$. As $\mathcal{O}(\lvert R \rvert)$ dominates $\mathcal{O}(\log \lvert R \rvert)$, we conclude that LA3P has a run time of $\mathcal{O}(\lvert R \rvert)$ in the worst case scenario.

Although the array division in the additive sum tree, i.e., taking the inverse of the priorities by multiplication, dramatically increases the computational complexity of vanilla PER and may question the feasibility of our approach, it can be overcome by Single Instruction, Multiple Data (SIMD) structure supported by CPUs introduced recently. Exceptionally, SIMD instructions perform the same operation, such as the simple array division in our case, on all cores in parallel. Fortunately, this is not the user's concern and can be executed implicitly by the CPU. Therefore, we believe that the computational burden of the LA3P framework will be significantly reduced by the additional computational efficiency offered by SIMD instructions.

\section{Experiments}
\subsection{Experimental Details}
With all the mentioned concepts combined, we investigate to what extent our prioritization framework can improve the performance of off-policy actor-critic methods in continuous control. Thus, we perform experiments to evaluate the effectiveness of LA3P on the standard suite of MuJoCo \citep{mujoco} and Box2D \citep{box2d} continuous control tasks interfaced by OpenAI Gym. We combine our method with the state-of-the-art actor-critic algorithms, Twin Delayed Deep Deterministic Policy Gradient (TD3) \citep{td3} and Soft Actor-Critic (SAC) \citep{sac}, which we benchmark against uniform sampling, PER, and the rule-based PER correction methods of LAP and MaPER. PAL combined with uniform sampling could also be used to compare our method, however, it has the same expected gradient as LAP, as stated by \cite{lap}. Thus, we would expect the same empirical performance.

Our implementation of the state-of-the-art algorithms closely follows the hyper-parameter setting and architecture outlined in the original papers. Particularly, we implement TD3 using the code from the author's GitHub repository\footnote{\url{https://github.com/sfujim/TD3}\label{td3_repo}}, which contains the fine-tuned version of the algorithm. The implementation of SAC is precisely based on the original paper. Unlike the paper, we include entropy tuning, as shown by \cite{sac_applications} to improve the algorithm's overall performance. Moreover, we add 25000 exploration time steps before the training to increase the data efficiency as indicated by \cite{td3}.

We use the LAP and PAL code in the author's GitHub repository\footnote{\url{https://github.com/sfujim/LAP-PAL}\label{lap_repo}} to implement the algorithm and our framework, which requires a few lines on top of the standard PER implementation. Moreover, we use the same repository for the PER implementation, which is based on proportional prioritization through sum trees. Specifically, LA3P is implemented by cascading uniform sampling combined with PAL and PER combined with LAP. For all experience replay sampling algorithms except for uniform, we use $\beta = 0.4$, while we set $\alpha = 0.6$ and $\alpha = 0.4$ for PER and LAP, respectively, as indicated in the original papers. Since we use LAP and PAL functions in our framework, we also set $\alpha = 0.4$ for LA3P. Finally, the code in the submission website\footnote{\url{https://openreview.net/forum?id=WuEiafqdy9H}\label{maper_repo}} of MaPER is used to implement the algorithm. Exact hyper-parameter settings, architecture, and implementation are broadly explained in Appendix \ref{app:exp_details}.

Each method is trained for a million steps over ten random seeds of network initialization, simulators, and dependencies. Every 1000 steps, each method is evaluated in a distinct evaluation environment (training seed + constant) for ten episodes, where no exploration and learning are performed. To construct the reported learning curves, we report the average of these ten evaluation episodes occurring at each evaluation interval. For easy reproducibility and fair evaluation, we did not modify the simulators' environment dynamics or reward functions. Computing infrastructure, e.g., hardware and software, used to produce the reported results are summarized in our repository\footref{our_repo}. Detailed experimental setup is provided in Appendix \ref{app:exp_details}. 

\subsection{Comparative Evaluation}
Learning curves for the set of OpenAI Gym continuous control benchmarks are reported in Figure \ref{fig:sac_results} and \ref{fig:td3_results} for the SAC and TD3 algorithms, respectively. For all tasks, we use uniform fraction value of $\lambda = 0.5$. Initially, we tested $\lambda = \{0.1, 0.3, 0.5, 0.7, 0.9\}$ on the Ant, HalfCheetah, Humanoid, and Walker2d tasks and found that $\lambda = 0.5$ produced the best results. In the next section, we also present a sensitivity analysis based on the $\lambda$ parameter. Empirical complexity analysis is provided in Appendix \ref{sec:emp_comp}.

We additionally report the average of the last ten evaluation returns, i.e., the level where the algorithms converge, in Table \ref{tab:eval}. Note that for some of the tasks, e.g., HalfCheetah, Hopper, and Walker2d, the baseline competing algorithms performed worse than what was reported in the original articles. This is due to the stochasticity of the simulators and used random seeds. Nonetheless, regardless of where the baselines converge, the performance disparity between the competing approaches would practically remain the same if we employed different sets of random seeds. As a result, our comparative analyses are fair by the deep RL benchmarking standards \citep{deep_rl_that_matters}.

\begin{figure*}[!t]
    \centering
    \begin{align*}
        &\text{{\orange} SAC + LA3P ($\lambda$ = 0.5)}  &&\text{{\green} SAC + PER} &&\text{{\brown} SAC + LAP} \\
        &\text{{\purple} SAC + MaPER} &&\text{{\blue} SAC + uniform}
    \end{align*}
	\subfigure{
		\includegraphics[width=1.55in, keepaspectratio]{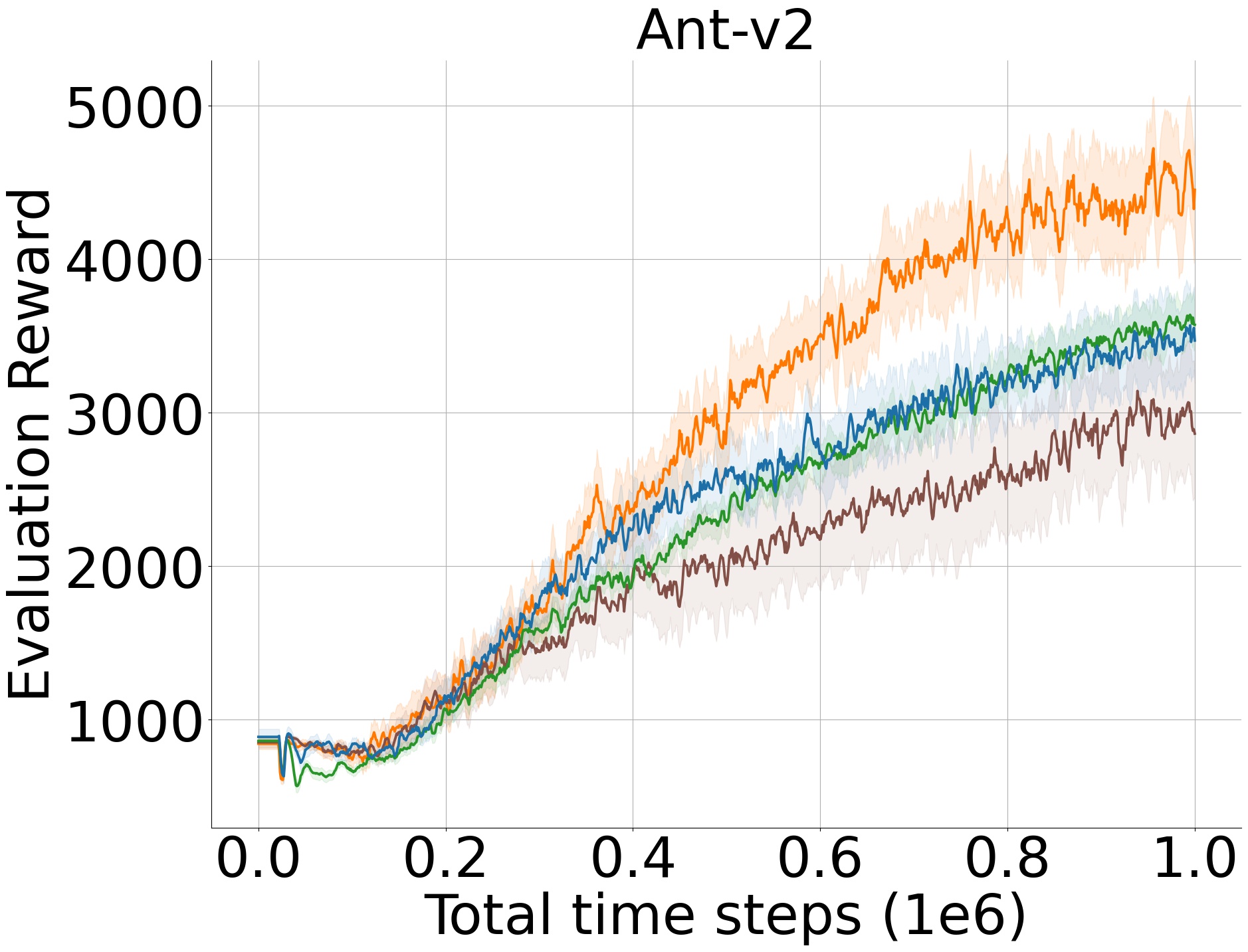}
		\includegraphics[width=1.55in, keepaspectratio]{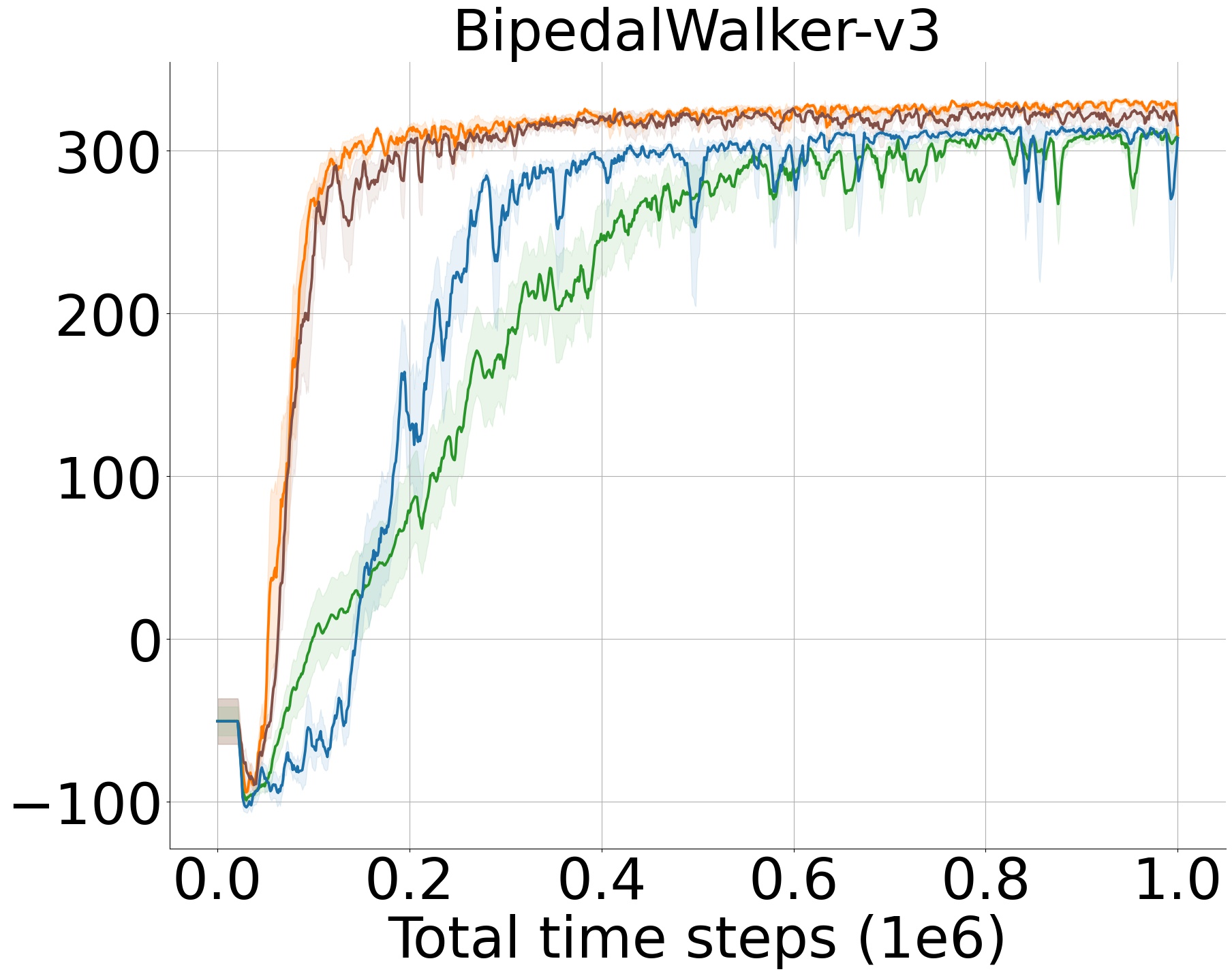}
		\includegraphics[width=1.55in, keepaspectratio]{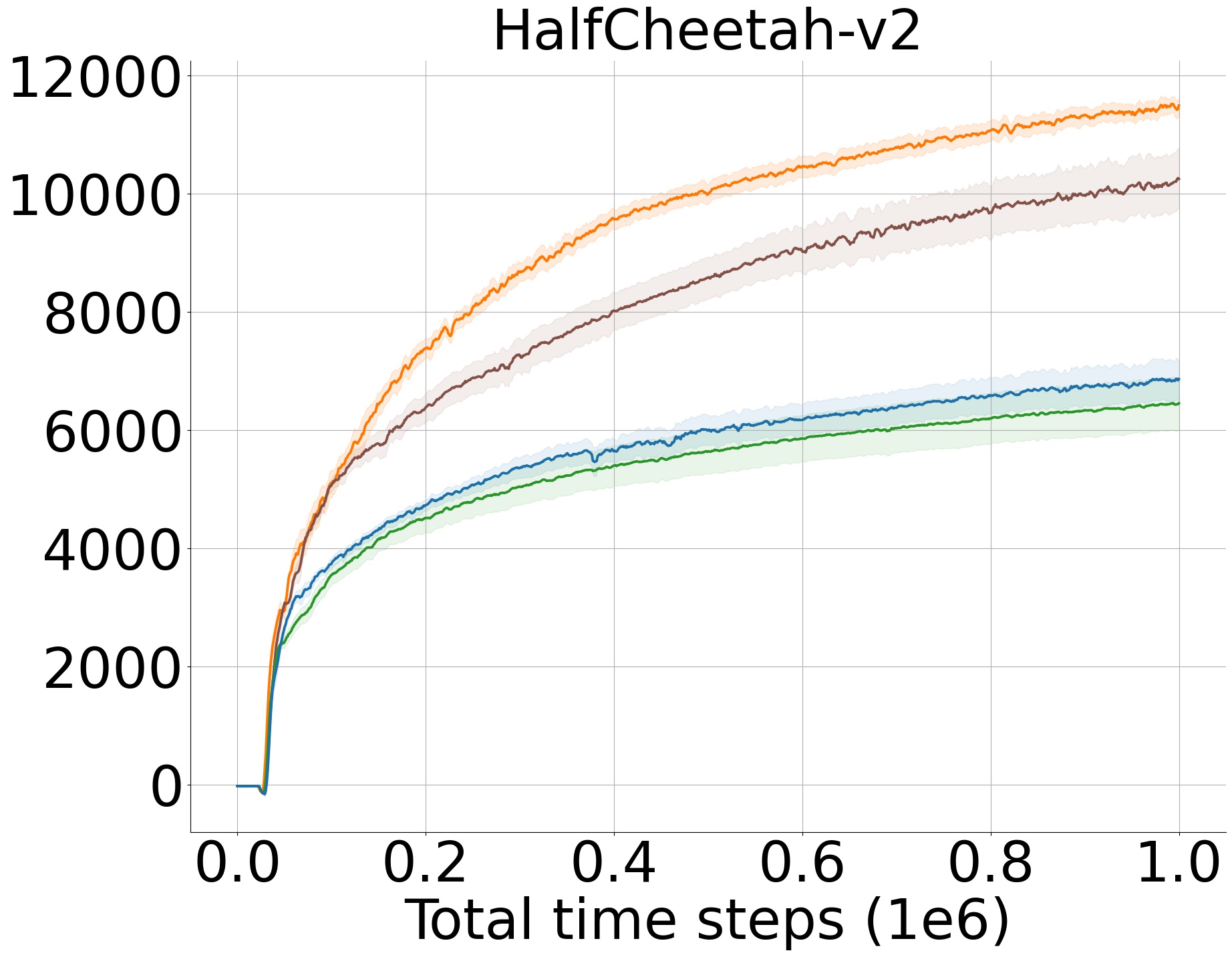}
		\includegraphics[width=1.55in, keepaspectratio]{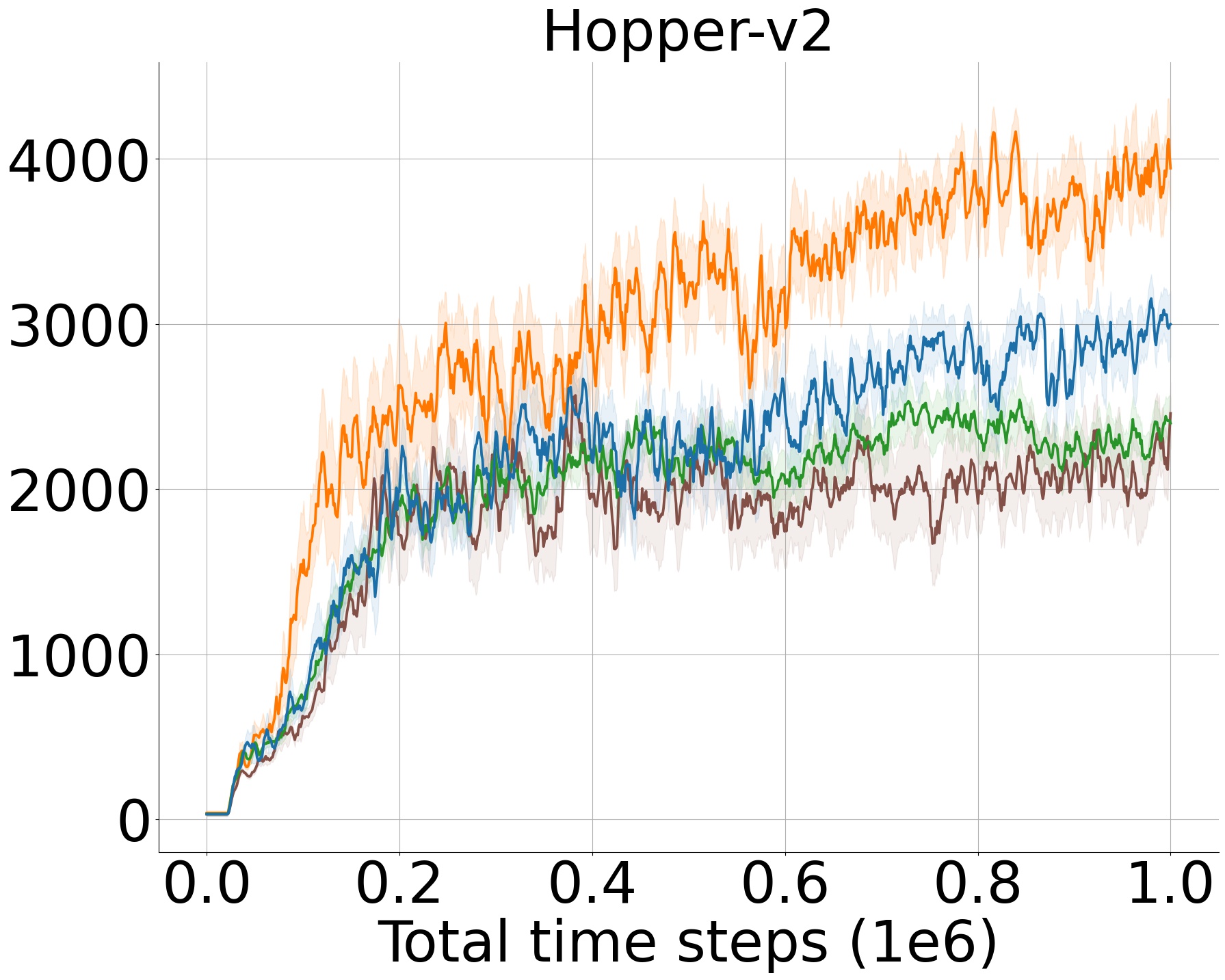}
	} \\
	\subfigure{
	    \includegraphics[width=1.55in, keepaspectratio]{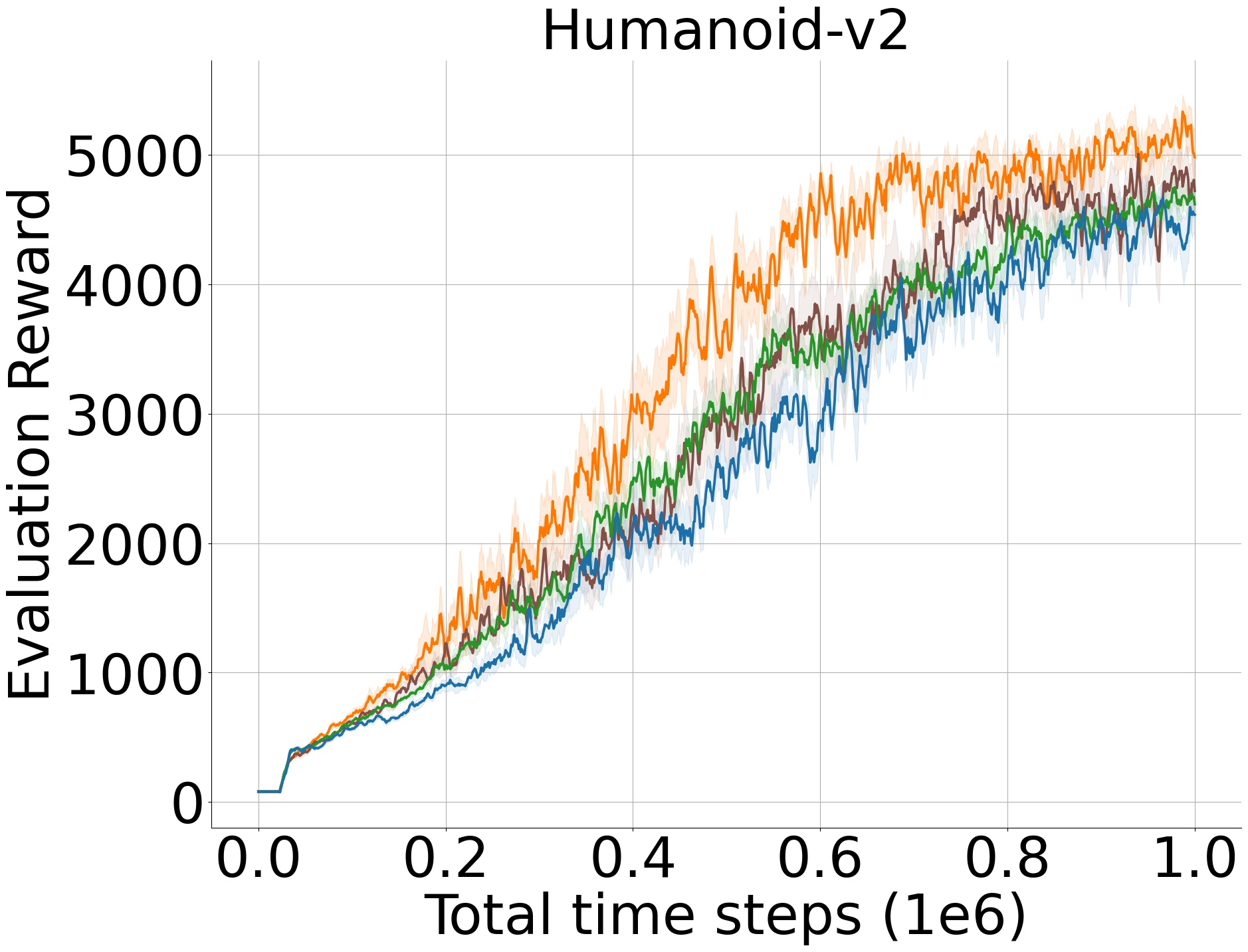}
		\includegraphics[width=1.55in, keepaspectratio]{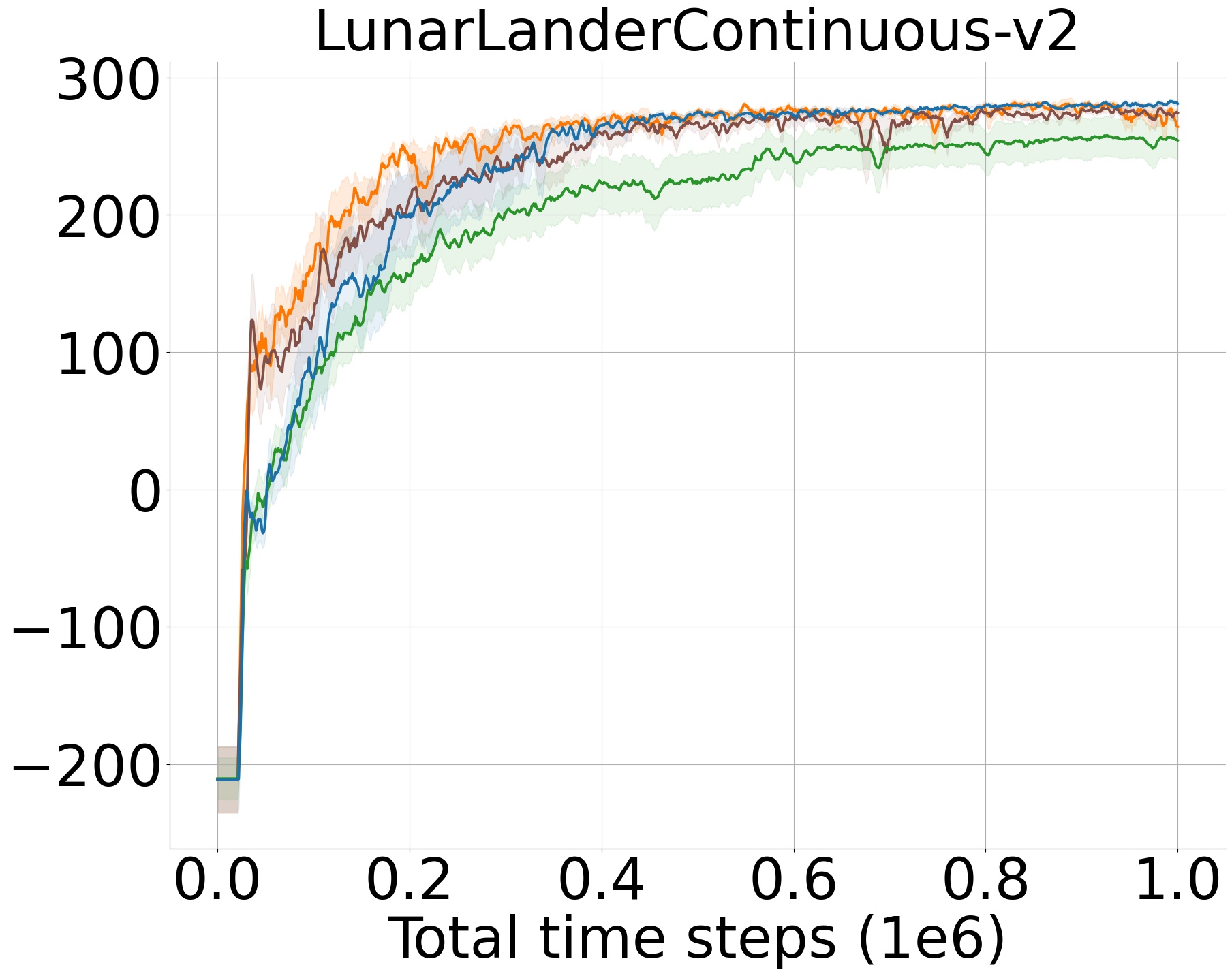}
		\includegraphics[width=1.55in, keepaspectratio]{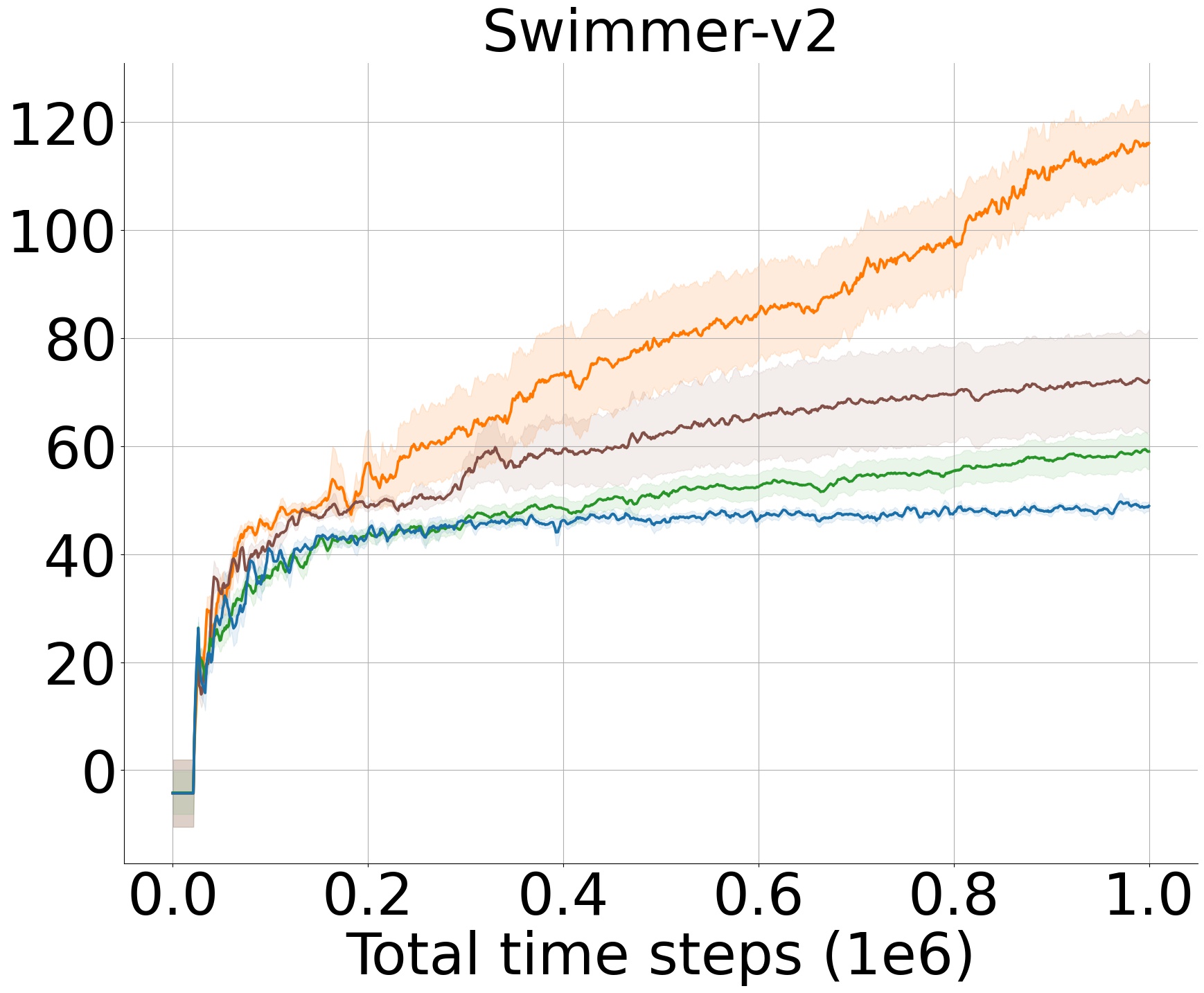}
		\includegraphics[width=1.55in, keepaspectratio]{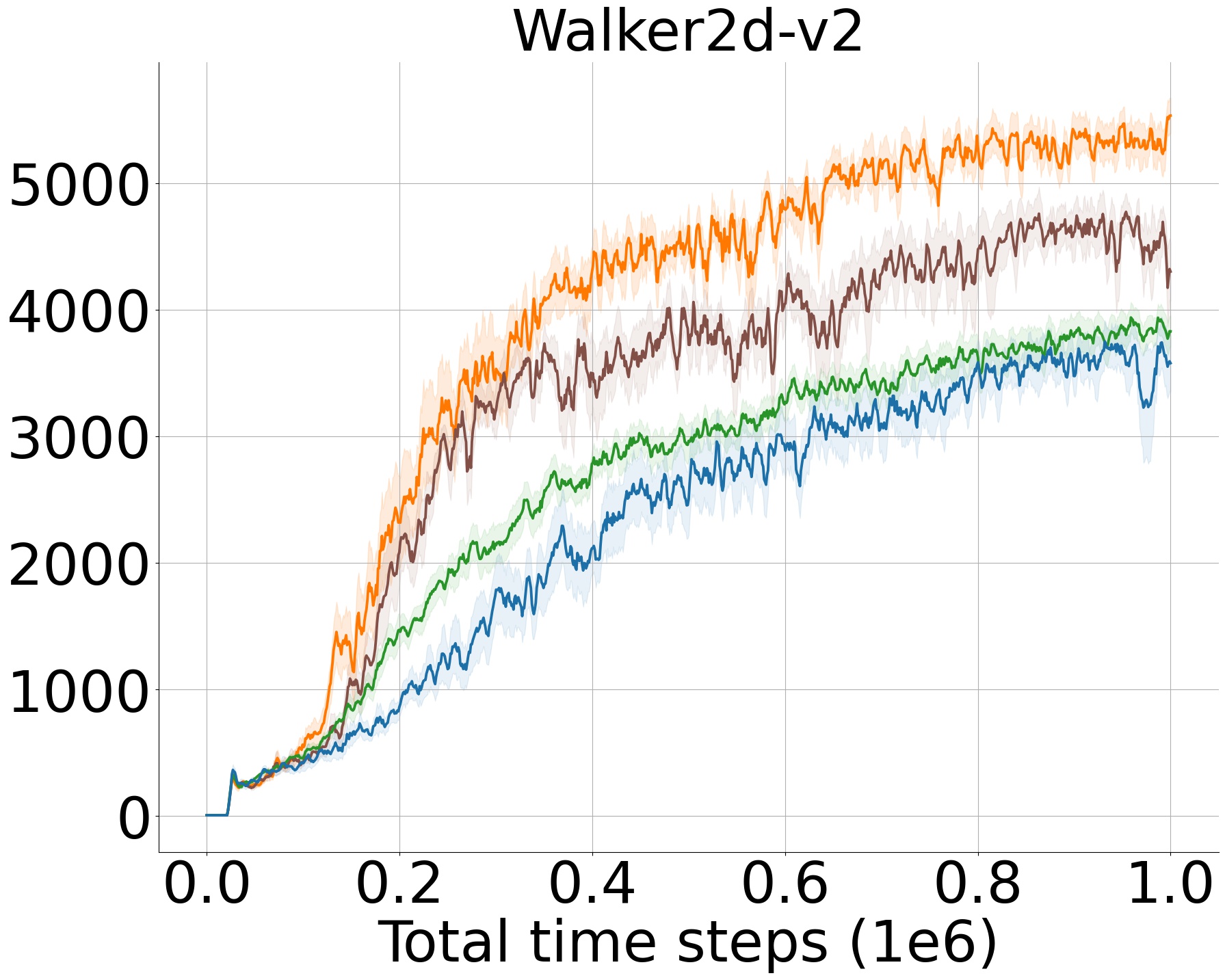}
    }
	\caption{Learning curves for the set of MuJoCo and Box2D continuous control tasks under the SAC algorithm. The shaded region represents a 95\% confidence interval over the trials. A sliding window of size 5 smoothes the curves for visual clarity.}
	\label{fig:sac_results}
\end{figure*}

\begin{figure*}[!hbt]
    \centering
    \begin{align*}
        &\text{{\orange} TD3 + LA3P ($\lambda$ = 0.5)}  &&\text{{\green} TD3 + PER} &&\text{{\brown} TD3 + LAP} \\
        &\text{{\purple} TD3 + MaPER} &&\text{{\blue} TD3 + uniform}
    \end{align*}
	\subfigure{
		\includegraphics[width=1.55in, keepaspectratio]{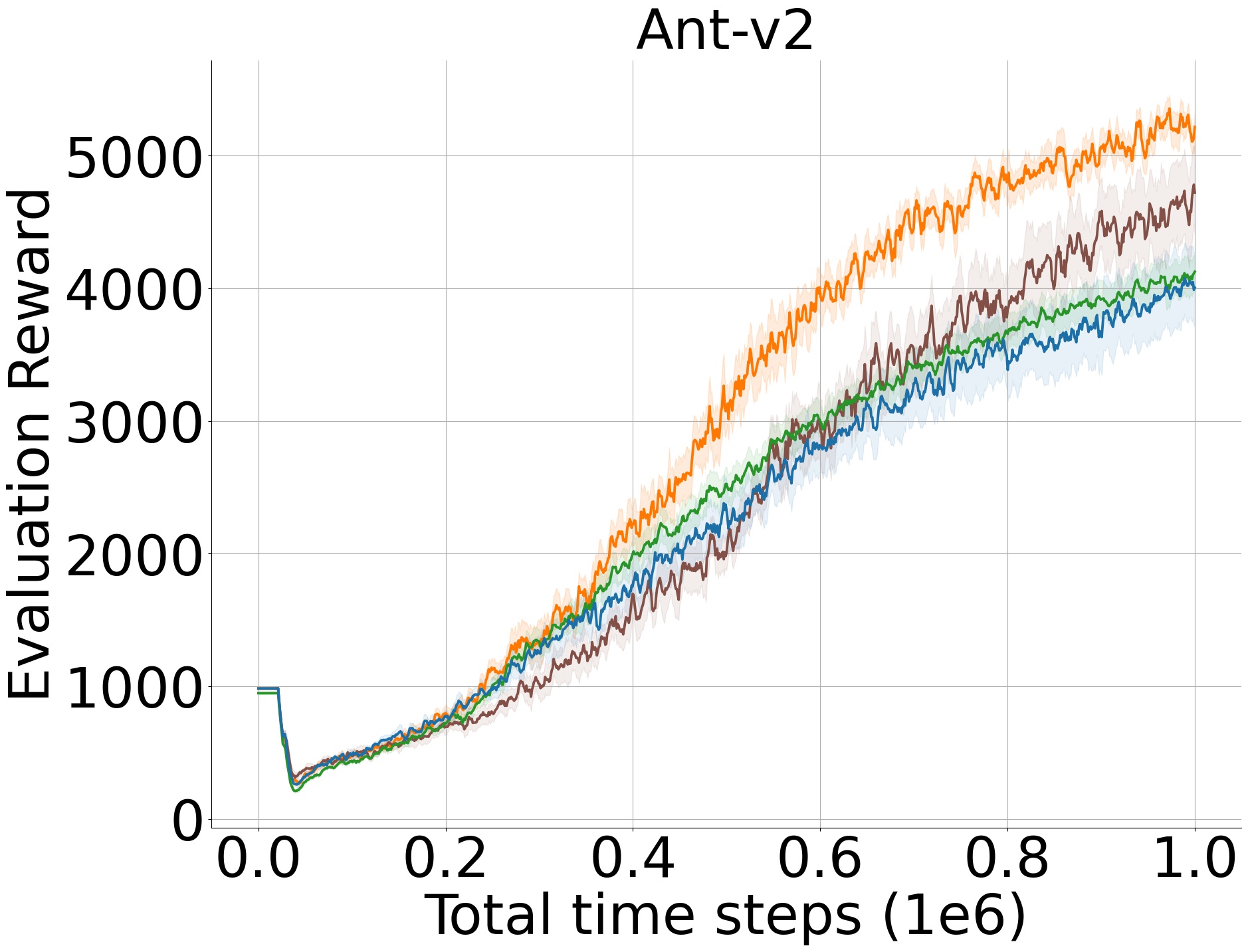}
		\includegraphics[width=1.55in, keepaspectratio]{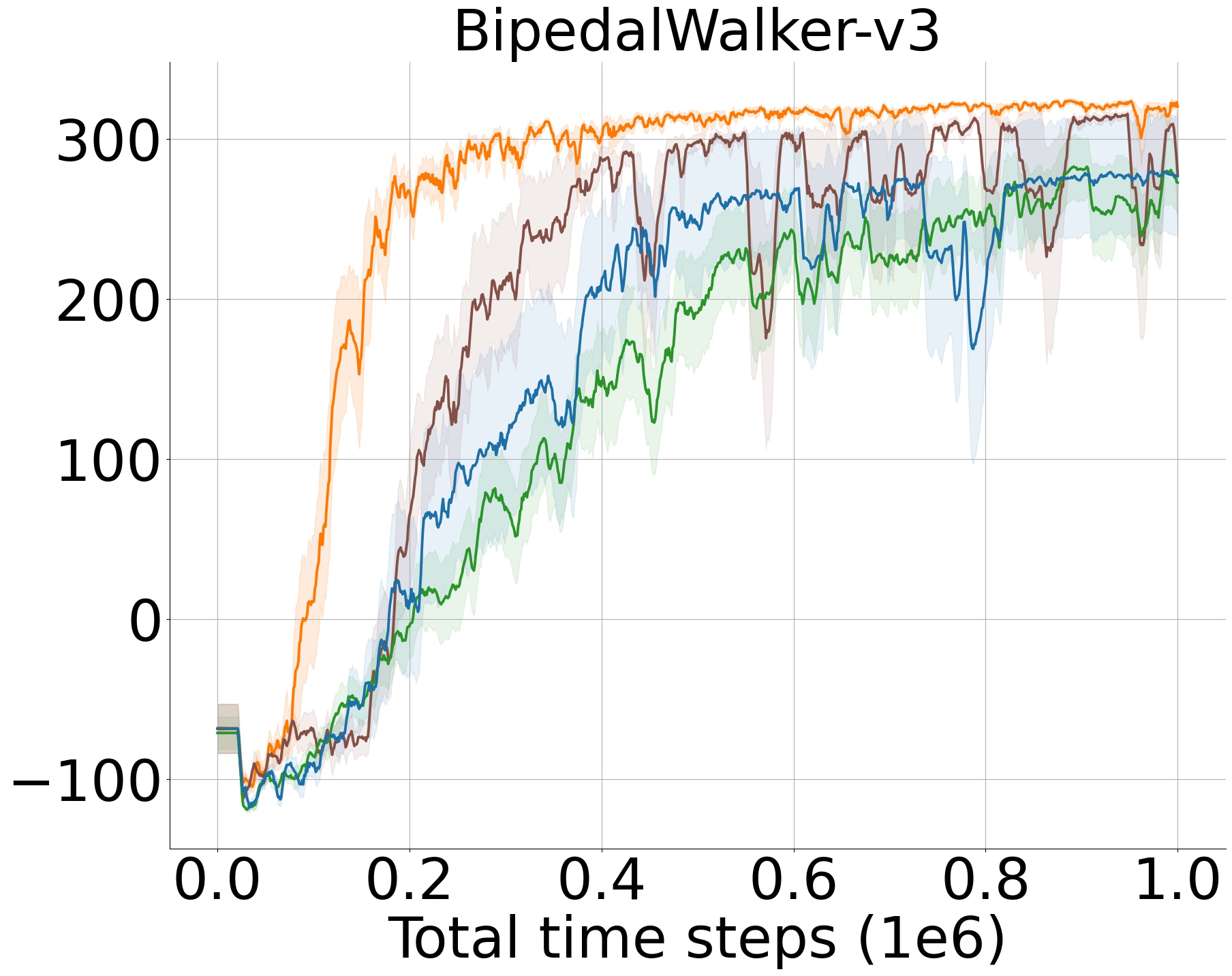}
		\includegraphics[width=1.55in, keepaspectratio]{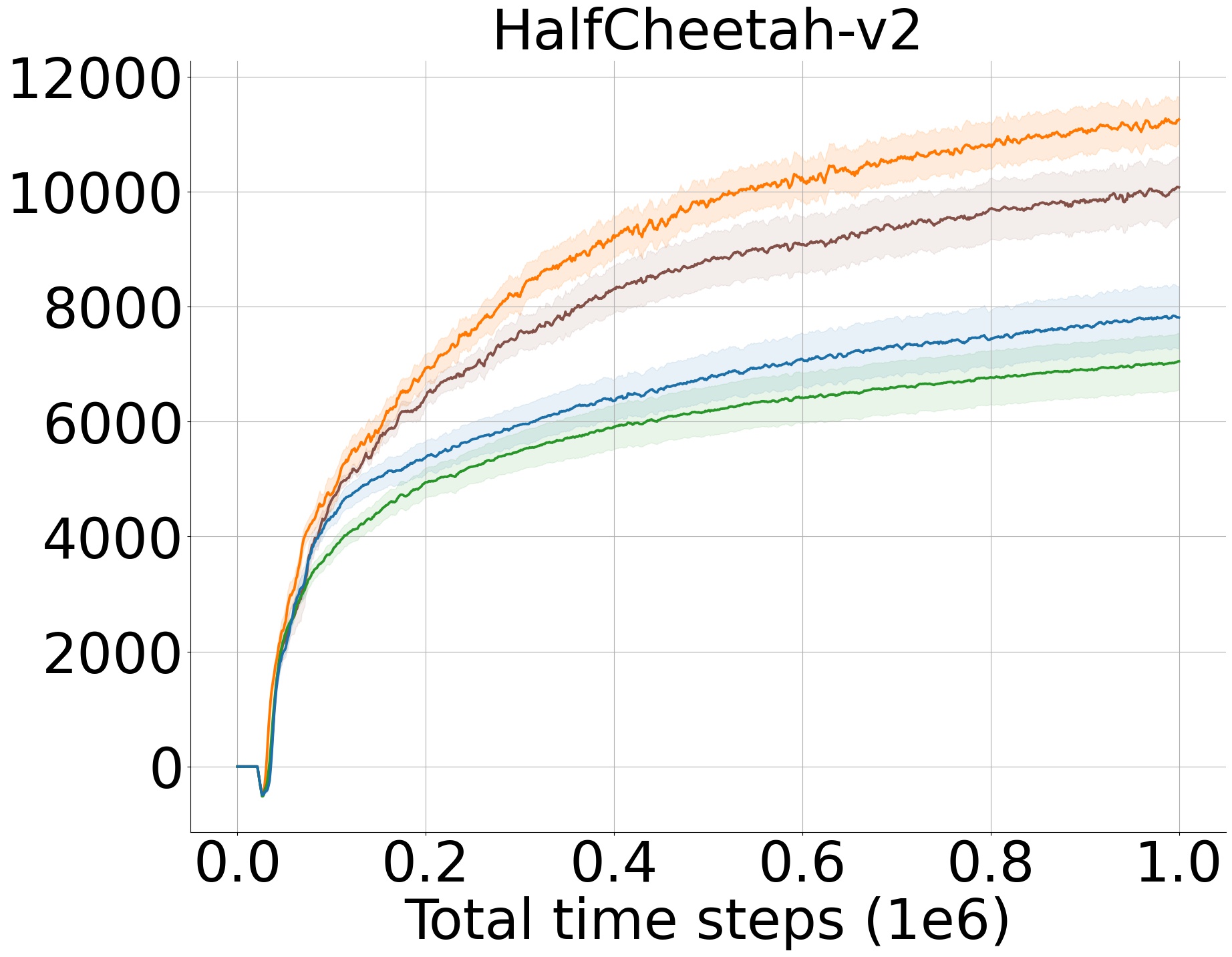}
		\includegraphics[width=1.55in, keepaspectratio]{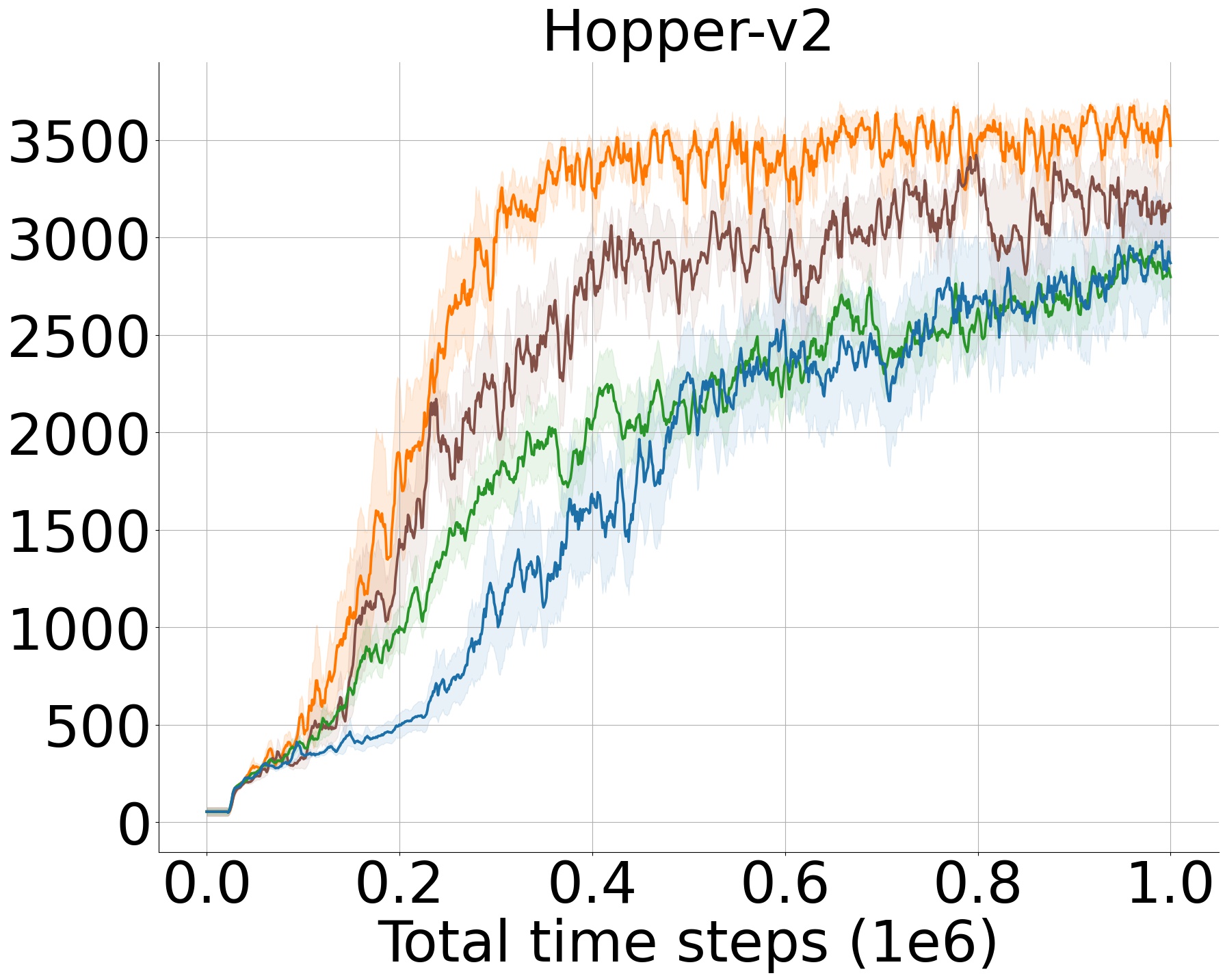}
	} \\
	\subfigure{
	    \includegraphics[width=1.55in, keepaspectratio]{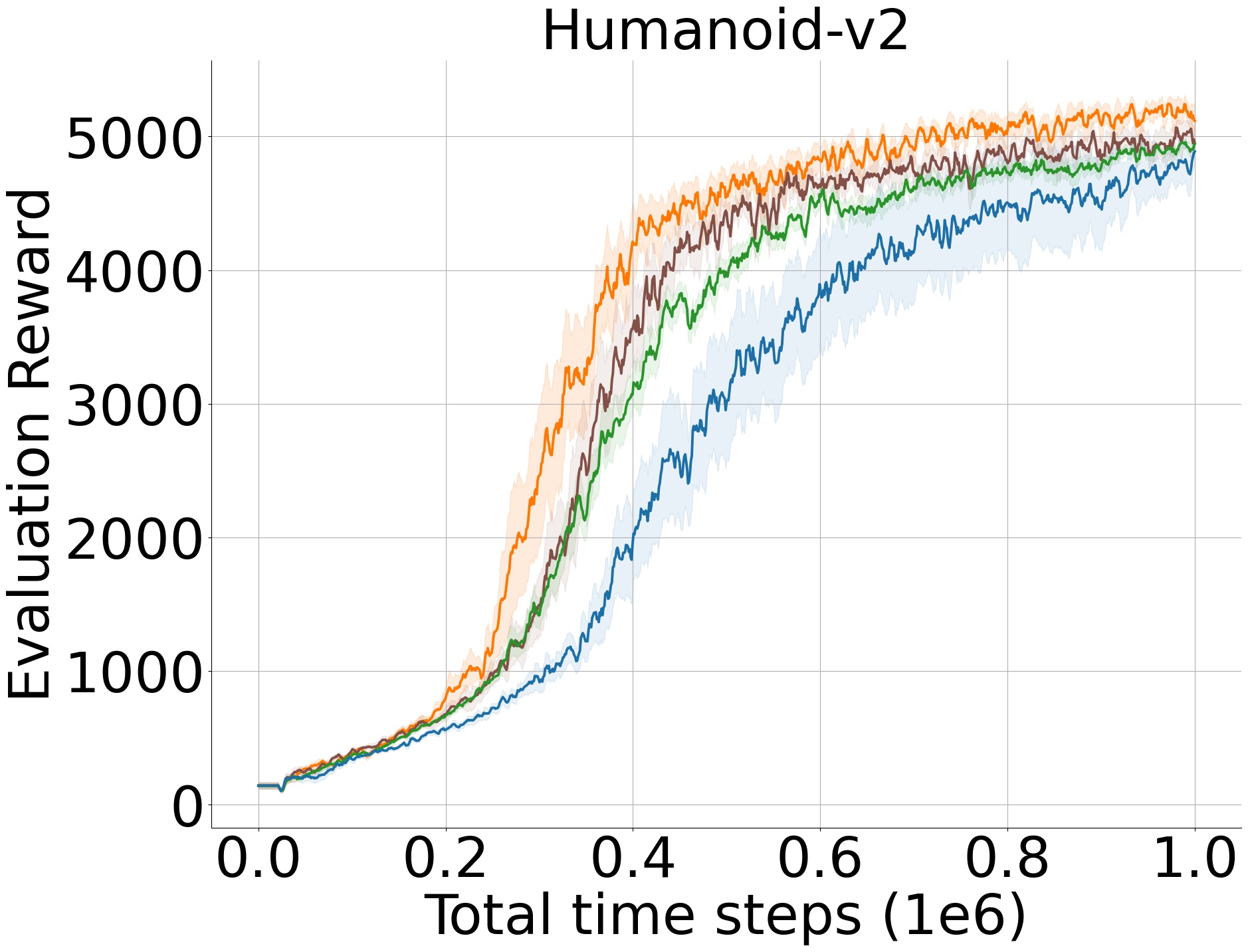}
		\includegraphics[width=1.55in, keepaspectratio]{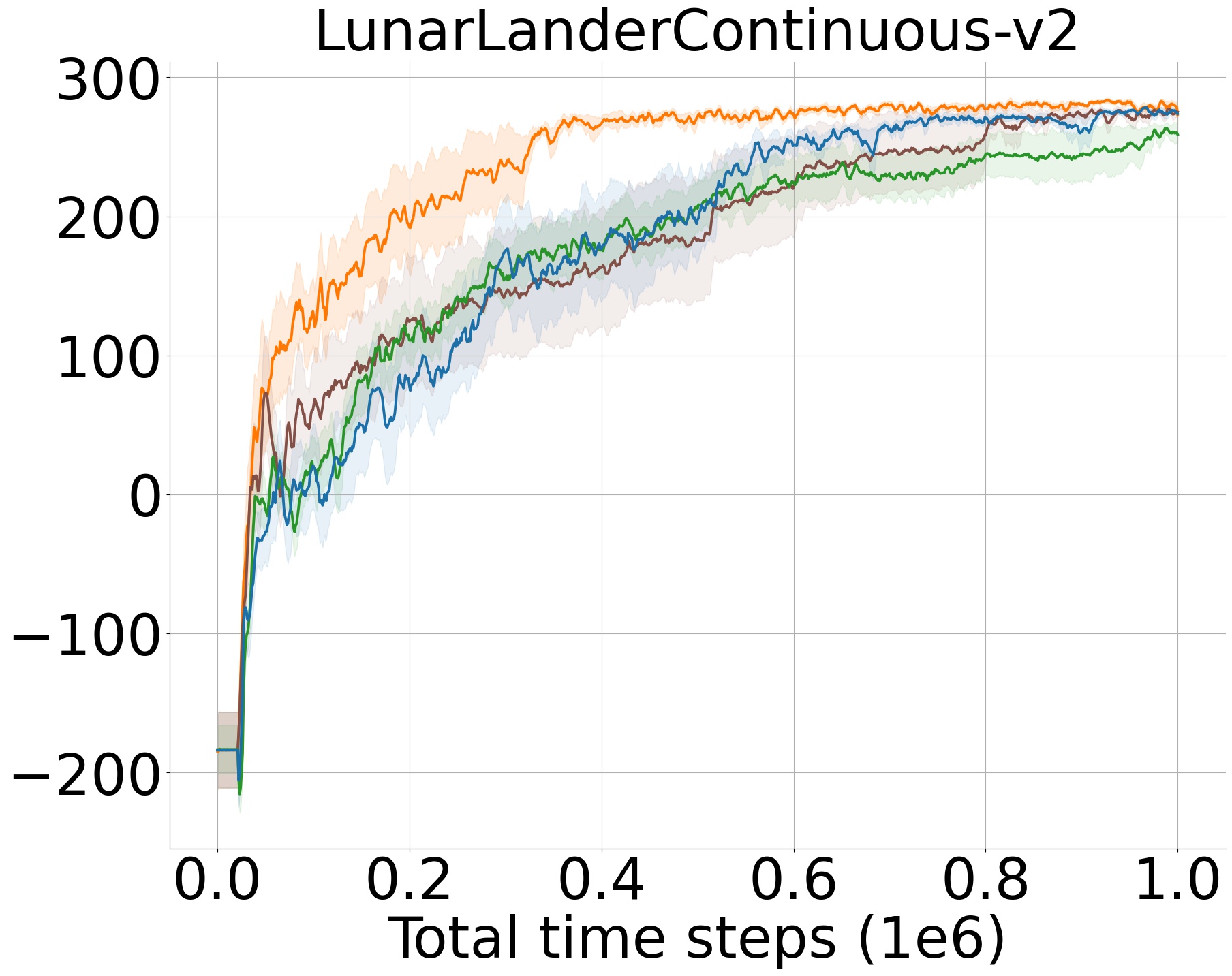}
		\includegraphics[width=1.55in, keepaspectratio]{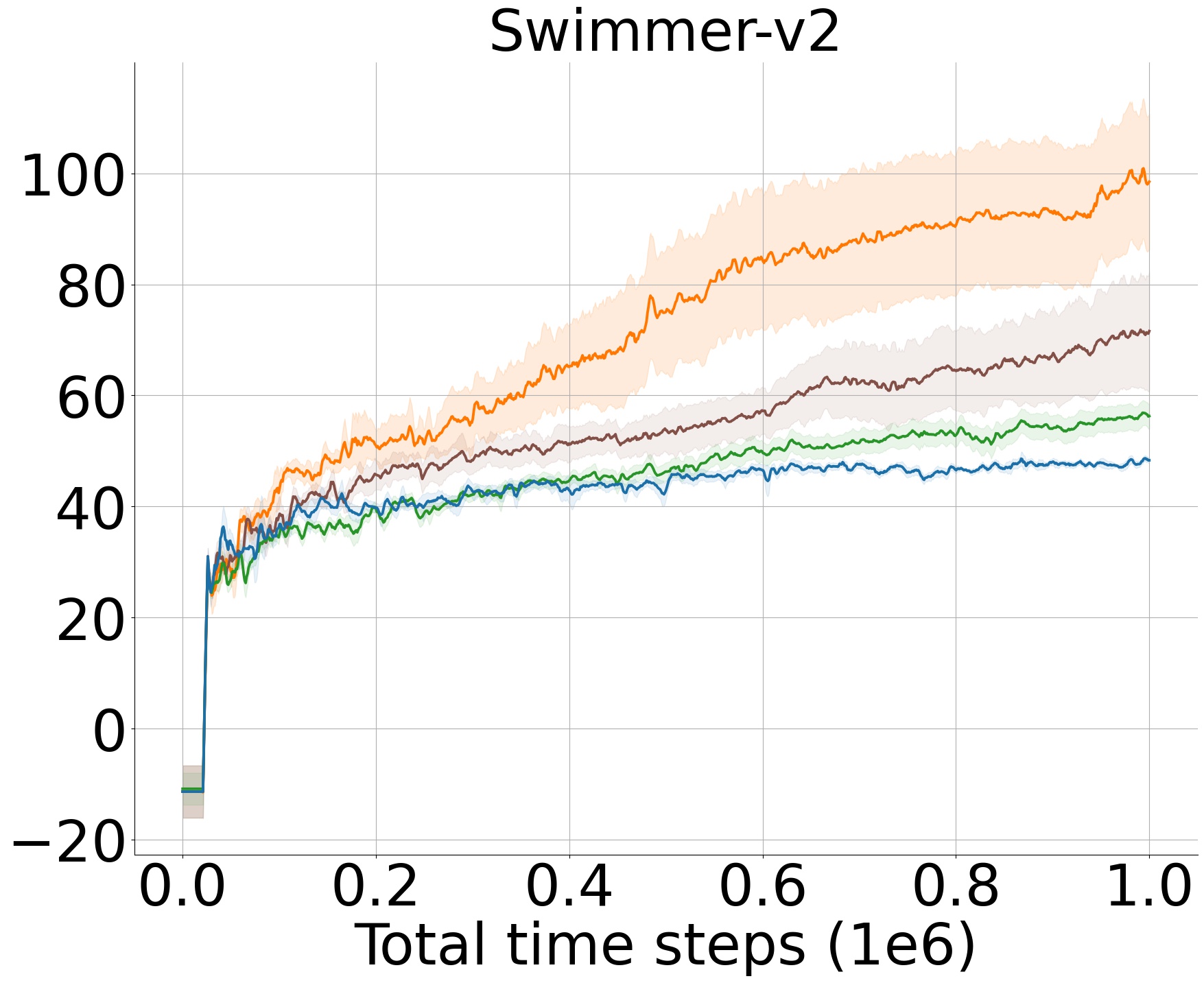}
		\includegraphics[width=1.55in, keepaspectratio]{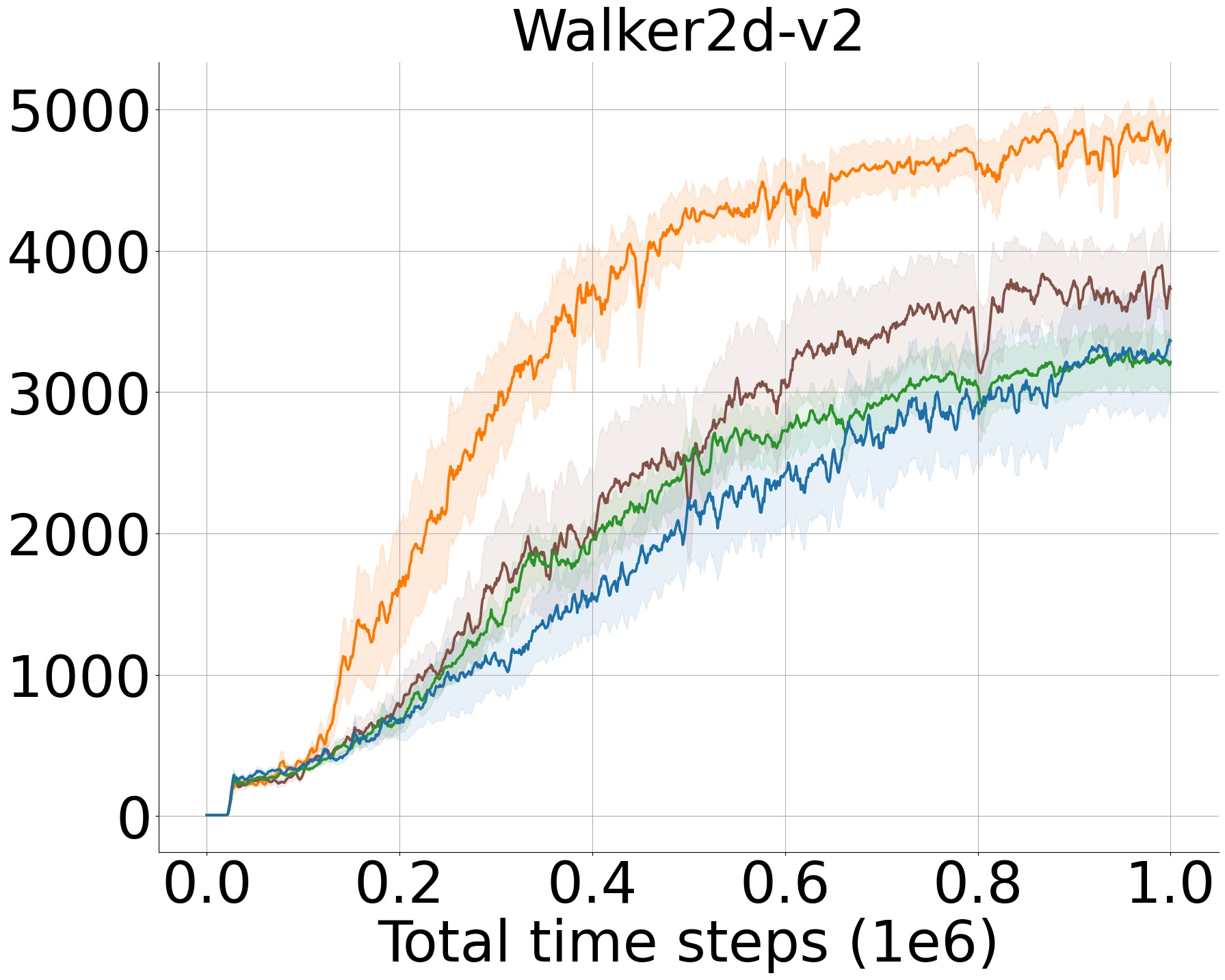}
    }
	\caption{Learning curves for the set of MuJoCo and Box2D continuous control tasks under the TD3 algorithm. The shaded region represents a 95\% confidence interval over the trials. A sliding window of size 5 smoothes the curves for visual clarity.}
	\label{fig:td3_results}
\end{figure*}

\label{sec:ab_stud}
\begin{sidewaystable*}[!hbpt]
    \centering
    \scalebox{0.8}{
    \begin{tabular}[hbpt]{lcccccccc}
    \toprule%
    \textbf{Method}  & \textbf{Ant} & \textbf{BipedalWalker} & \textbf{HalfCheetah} & \textbf{Hopper} & \textbf{Humanoid} & \textbf{LunarLanderContinuous} & \textbf{Swimmer} & \textbf{Walker2d}  \\        
    \midrule
    SAC + LA3P & \textbf{4539.10 $\pm$ 810.23} & 318.91 $\pm$ 27.87 & \textbf{11485.46 $\pm$ 339.89} & \textbf{3917.77 $\pm$ 727.46} & \textbf{5094.80 $\pm$ 518.17} & 269.06 $\pm$ 16.11 & \textbf{115.83 $\pm$ 14.59} & \textbf{5449.14 $\pm$ 370.56} \\
    SAC + LAP & 2934.10 $\pm$ 907.89 & \textbf{320.05 $\pm$ 15.72} & 10223.23 $\pm$ 1016.31 & 2324.29 $\pm$ 532.27 & 4726.08 $\pm$ 632.85 & \textbf{272.66 $\pm$ 11.19} & 72.02 $\pm$ 18.49 & 4403.11 $\pm$ 742.64 \\
    SAC + MaPER & 3675.63 $\pm$ 376.33 & 308.03 $\pm$ 20.05 & 8270.91 $\pm$ 326.49 & 2840.91 $\pm$ 289.27 & 4821.05 $\pm$ 272.15 & 265.22 $\pm$ 17.96 & 70.99 $\pm$ 6.86 & 4243.83 $\pm$ 283.88 \\
    SAC + PER & 3605.65 $\pm$ 333.07 & 306.44 $\pm$ 10.54 & 6444.57 $\pm$ 897.16 & 2420.24 $\pm$ 335.34 & 4652.81 $\pm$ 279.68 & 254.92 $\pm$ 28.54 & 59.17 $\pm$ 6.52 & 3830.46 $\pm$ 295.52 \\
    SAC + Uniform & 3519.85 $\pm$ 634.14 & 290.44 $\pm$ 51.82 & 6845.84 $\pm$ 686.79 & 3025.97 $\pm$ 480.88 & 4567.03 $\pm$ 477.81 & 281.61 $\pm$ 4.58 & 48.76 $\pm$ 2.23 & 3609.33 $\pm$ 516.24 \\ \\
    
    TD3 + LA3P & \textbf{5197.46 $\pm$ 377.34} & \textbf{321.17 $\pm$ 7.03} & \textbf{11225.14 $\pm$ 811.57} & \textbf{3563.00 $\pm$ 279.32} & \textbf{5131.11 $\pm$ 250.77} & \textbf{276.60 $\pm$ 15.13} & \textbf{99.33 $\pm$ 25.04} & \textbf{4776.68 $\pm$ 424.46} \\
    TD3 + LAP & 4653.16 $\pm$ 701.20 & 293.68 $\pm$ 45.51 & 10052.97 $\pm$ 1056.09 & 3145.94 $\pm$ 597.24 & 4998.46 $\pm$ 279.67 & 274.22 $\pm$ 10.36 & 71.33 $\pm$ 20.76 & 3700.36 $\pm$ 766.17 \\
    TD3 + MaPER & 4161.24 $\pm$ 237.26 & 306.43 $\pm$ 27.86 & 8975.32 $\pm$ 604.96 & 3027.40 $\pm$ 278.29 & 4938.53 $\pm$ 140.22 & 267.60 $\pm$ 9.52 & 61.80 $\pm$ 7.46 & 3410.70 $\pm$ 341.82 \\
    TD3 + PER & 4103.93 $\pm$ 287.86 & 275.49 $\pm$ 40.91 & 7035.20 $\pm$ 984.17 & 2802.30 $\pm$ 277.97 & 4916.40 $\pm$ 112.79 & 259.78 $\pm$ 13.93 & 56.56 $\pm$ 4.52 & 3221.13 $\pm$ 420.88 \\
    TD3 + Uniform & 4029.16 $\pm$ 576.29 & 277.14 $\pm$ 74.78 & 7824.56 $\pm$ 1091.80 & 2857.00 $\pm$ 584.39 & 4802.12 $\pm$ 310.07 & 274.95 $\pm$ 5.70 & 48.49 $\pm$ 1.30 & 3312.46 $\pm$ 832.18 \\
    \bottomrule
    \end{tabular}
    }
    \caption{Average return of last 10 evaluations over 10 trials
    of 1 million time steps. $\pm$ captures a 95\% confidence interval over the trials. Bold values represent the maximum under each baseline algorithm and environment.}
    \label{tab:eval}
\end{sidewaystable*}

From our evaluation results, we observe that LA3P matches or outperforms the competing approaches in all tasks and baseline off-policy actor-critic algorithms tested. However, for the BipedalWalker and LunarLanderContinuous environments under the SAC algorithm, we find that LAP obtains slightly larger rewards than our method. Nevertheless, these scores can be practically counted as the same, and as the learning curves show, LA3P could attain faster convergence. In trivial environments such as LunarLanderContinuous, a minor improvement is achieved. However, in the Swimmer environment where no algorithm could converge, the performance improvement offered by our modifications becomes more prominent. Furthermore, we also observe a substantial improvement by LA3P in the HalfCheetah environment over the prior approaches. As HalfCheetah and Swimmer are regarded as ``stable'', i.e., episodes terminate only if the pre-specified number of time steps is reached, they require long horizons to be simulated. Therefore, as stated by \cite{lap}, the benefit of a corrected prioritization scheme is more observable in environments with longer horizons. 

Additionally, we confirm previous empirical studies, e.g., \citep{lap,maper}, which found that PER provides no benefits when added to off-policy continuous control algorithms, and performance is usually degraded. While this is attributed to the use of MSE by \cite{lap}, prioritization with corrected loss function, i.e., LAP, appears to have little impact in Ant, Hopper, and Walker2d compared to LA3P. In fact, the SAC algorithm is underperformed in Ant and Hopper. This result is consistent with our theoretical analysis made in Corollary \ref{cor:diverge_grad}, which shows that optimizing the actor network with transitions corresponding to large TD errors can cause the approximate policy gradient to diverge from the one computed under the optimal Q-function. In these environments, learning curves demonstrate that the performance gain offered by LA3P primarily comes from the inverse prioritized sampling for the actor network. This suggests that the inverse sampling in the LA3P framework plays a more significant role than the employed LAP and PAL functions. Hence, a combined solution, inverse sampling with corrected loss functions, is superior.

Finally, we notice that the performance of MaPER is not promising as it acquires slight improvements over PER. We attribute such a poor performance mainly to the model prediction structure of the algorithm. As we previously discussed, the MaPER algorithm focuses on new learnable features driven by the components in model-based RL to calculate the scores on experiences since critic networks often under- or overestimate Q-values. However, the Clipped Double Q-learning algorithm proposed by \cite{td3} already solves the issues with inaccurate Q-value estimates, which is already employed in SAC and TD3. Therefore, we conclude that the main drawback of PER is not the inaccurate Q-value estimates used in the priority calculations but the biased loss function and training the actor network with large TD error transitions. In addition, the model prediction module in MaPER decreases the convergence rate yet, brings notable stability to the learning. Nonetheless, the resulting performance is not considerable. Consequently, we believe that LA3P is a preferable and comprehensive way of overcoming the underlying issues of PER in continuous action domains. 

\subsection{Ablation Studies}
To better understand the contribution of each component in LA3P, we conduct an ablation study. The LA3P algorithm introduces several modifications to PER. In summary, LA3P consists of: (a) inverse sampling for the actor, (b) uniform sampling for the actor and critic networks to share a set of transitions, (c) the LAP function applied to the prioritized transitions, (d) the PAL function applied to the uniformly sampled transitions.

We evaluate and discuss the resulting performances when removing each of these components. In addition, we test the performance of LA3P when the shared transitions are low TD error experiences instead of uniformly sampled ones to demonstrate the mentioned decreased data efficiency. Finally, we perform a sensitivity analysis for the $\lambda$ parameter. We do not remove the inverse sampling as it is the backbone of our algorithm, that is, eliminating the inverse sampling for the actor network would not relate to any of the modifications introduced by this work as it would be just a mixture of uniform and prioritized sampling. Moreover, we choose four challenging environments with different characteristics for a comprehensive inference. As described by \cite{deep_rl_that_matters} and \cite{lap}, we consider the stable environment of HalfCheetah, the unstable environment of Walker2d, and the high dimensional and most challenging environments of Ant and Humanoid. 

Table \ref{table:ablation} presents the average of the last ten evaluation returns over ten trials, i.e., the level where the tested settings converge, for our ablation studies and sensitivity analysis, and the corresponding learning curves are depicted by Figure \ref{fig:ablation} and \ref{fig:fraction}, respectively. The same experimental setup is used to perform the ablation studies, and $\lambda = 0.5$ is used for all experiments unless otherwise stated. Note that $\lambda = 0.0$ yields the LA3P setting without shared set of transitions and the evaluation results of which are already provided in Table \ref{table:ablation} and Figure \ref{fig:ablation}. In addition, $\lambda = 1.0$ corresponds to uniform sampling, which we already compare against LA3P in Figure \ref{fig:sac_results}, \ref{fig:td3_results}, and Table \ref{tab:eval}.  

\begin{table}[!hbt]
\begin{center}
    \begin{tabular}{l cccc}
        \toprule
        \textbf{Setting} & \textbf{Ant} & \textbf{HalfCheetah} & \textbf{Humanoid} & \textbf{Walker2d} \\
        \midrule
        LA3P (complete) & \textbf{5197.46 $\pm$ 377.34} & \textbf{11225.14 $\pm$ 811.57} & \textbf{5131.11 $\pm$ 250.77} & \textbf{4776.68 $\pm$ 424.46} \\
        w/ low TD error & 3485.06 $\pm$ 946.02 & 10992.05 $\pm$ 485.02 & 3938.13 $\pm$ 1418.13 & 4438.29 $\pm$ 398.53 \\
        w/o LAP & 3408.55 $\pm$ 596.73 & 4580.27 $\pm$ 257.91 & 3585.06 $\pm$ 950.30 & 3262.49 $\pm$ 388.56 \\
        w/o PAL & 3975.29 $\pm$ 1207.92 & 7560.50 $\pm$ 774.59 & 4879.43 $\pm$ 352.00 & 4543.92 $\pm$ 528.66 \\
        w/o shared transitions & 4431.87 $\pm$ 787.78 & 10483.11 $\pm$ 600.52 & 5058.69 $\pm$ 183.93 & 4254.55 $\pm$ 390.57 \\ \\
        
        $\lambda = 0.1$ & 3768.74 $\pm$ 1079.59 & 11203.83 $\pm$ 576.07 & 4695.52 $\pm$ 338.39 & 4562.77 $\pm$ 405.97 \\
        $\lambda = 0.3$ & 4903.34 $\pm$ 467.54 & 10460.52 $\pm$ 948.60 & 4528.28 $\pm$ 1148.43 & 4425.08 $\pm$ 448.86 \\
        $\lambda = 0.5$ & \textbf{5197.46 $\pm$ 377.34} & \textbf{11225.14 $\pm$ 811.57} & \textbf{5131.11 $\pm$ 250.77} & \textbf{4776.68 $\pm$ 424.46} \\
        $\lambda = 0.7$ & 4383.97 $\pm$ 924.90 & 10656.92 $\pm$ 750.93 & 4874.50 $\pm$ 267.60 & 4253.76 $\pm$ 904.98 \\
        $\lambda = 0.9$ & 3882.08 $\pm$ 1157.63 & 10547.85 $\pm$ 890.96 & 3757.90 $\pm$ 1403.43 & 4545.97 $\pm$ 593.65 \\
        \bottomrule
    \end{tabular}
\end{center}
\caption{Average return over the last 10 evaluations over 10 trials
of 1 million time steps, comparing ablation over LA3P under low TD error shared transitions, LA3P without the LAP function, LA3P without the PAL function, LA3P without the shared set of transitions, and LA3P under $\lambda = \{0.1, 0.3, 0.5, 0.7, 0.9\}$. $\pm$ captures a 95\% confidence interval over the trials. Bold values represent the maximum under each environment.}
\label{table:ablation}
\end{table}

\begin{figure*}[!hbt]
    \centering
    \begin{align*}
        &\text{{\orange} LA3P (complete)}  &&\text{{\blue} LA3P w/ low TD error shared transitions} &&\text{{\brown} LA3P w/o LAP} \\
        &\text{{\green} LA3P w/o PAL} &&\text{{\pink} LA3P w/o shared transitions}
    \end{align*}
	\subfigure{
		\includegraphics[width=1.55in, keepaspectratio]{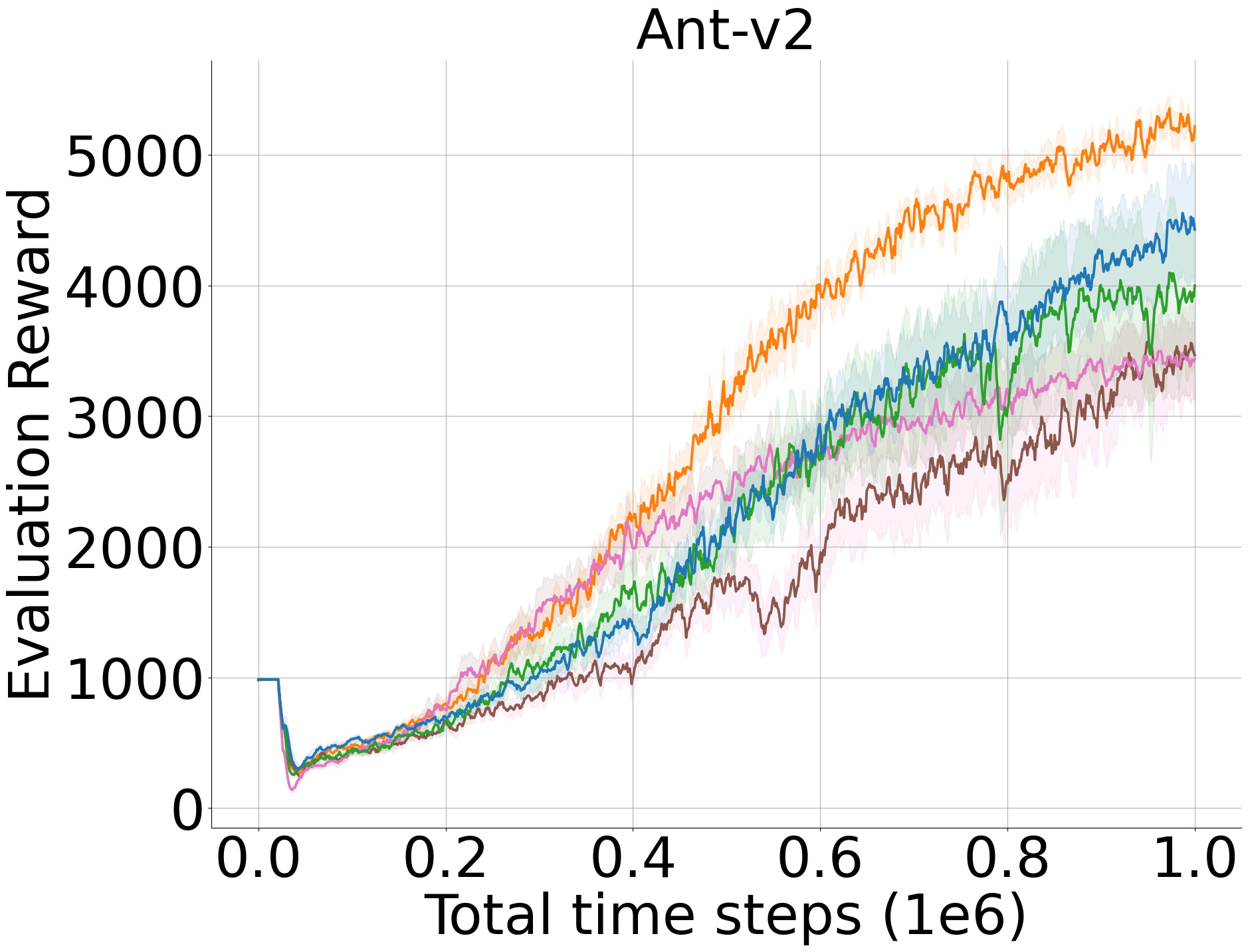}
		\includegraphics[width=1.55in, keepaspectratio]{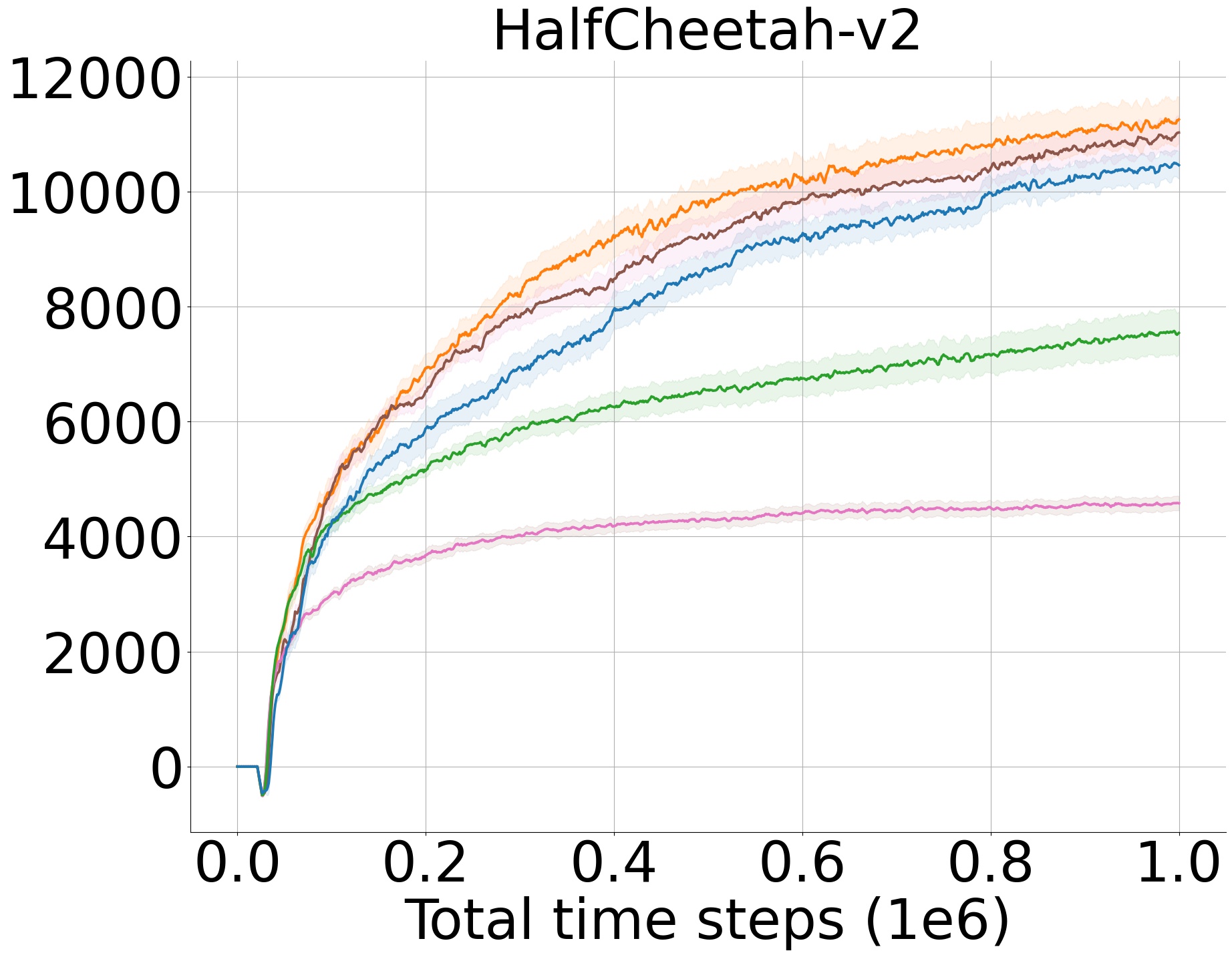}
		\includegraphics[width=1.55in, keepaspectratio]{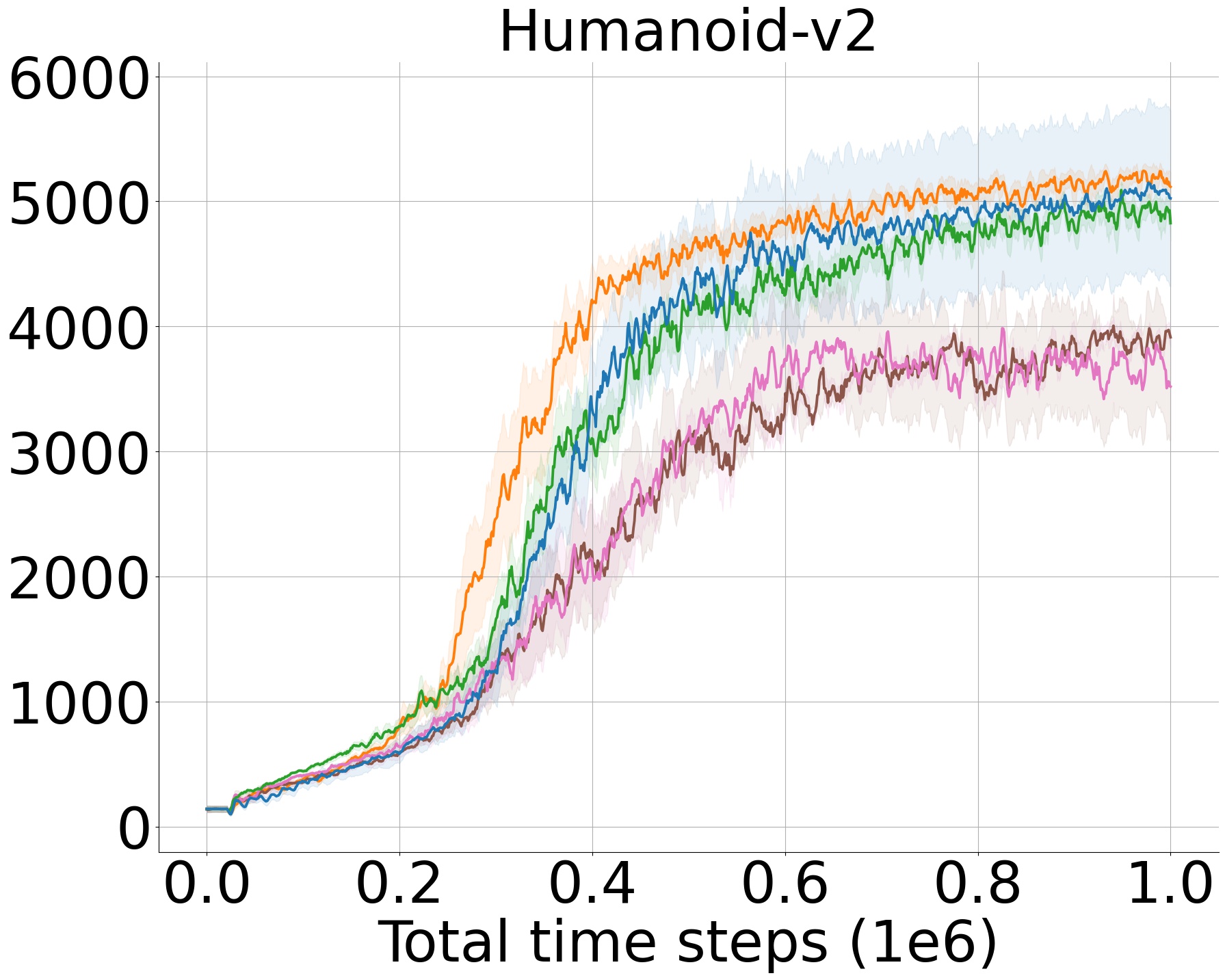}
		\includegraphics[width=1.55in, keepaspectratio]{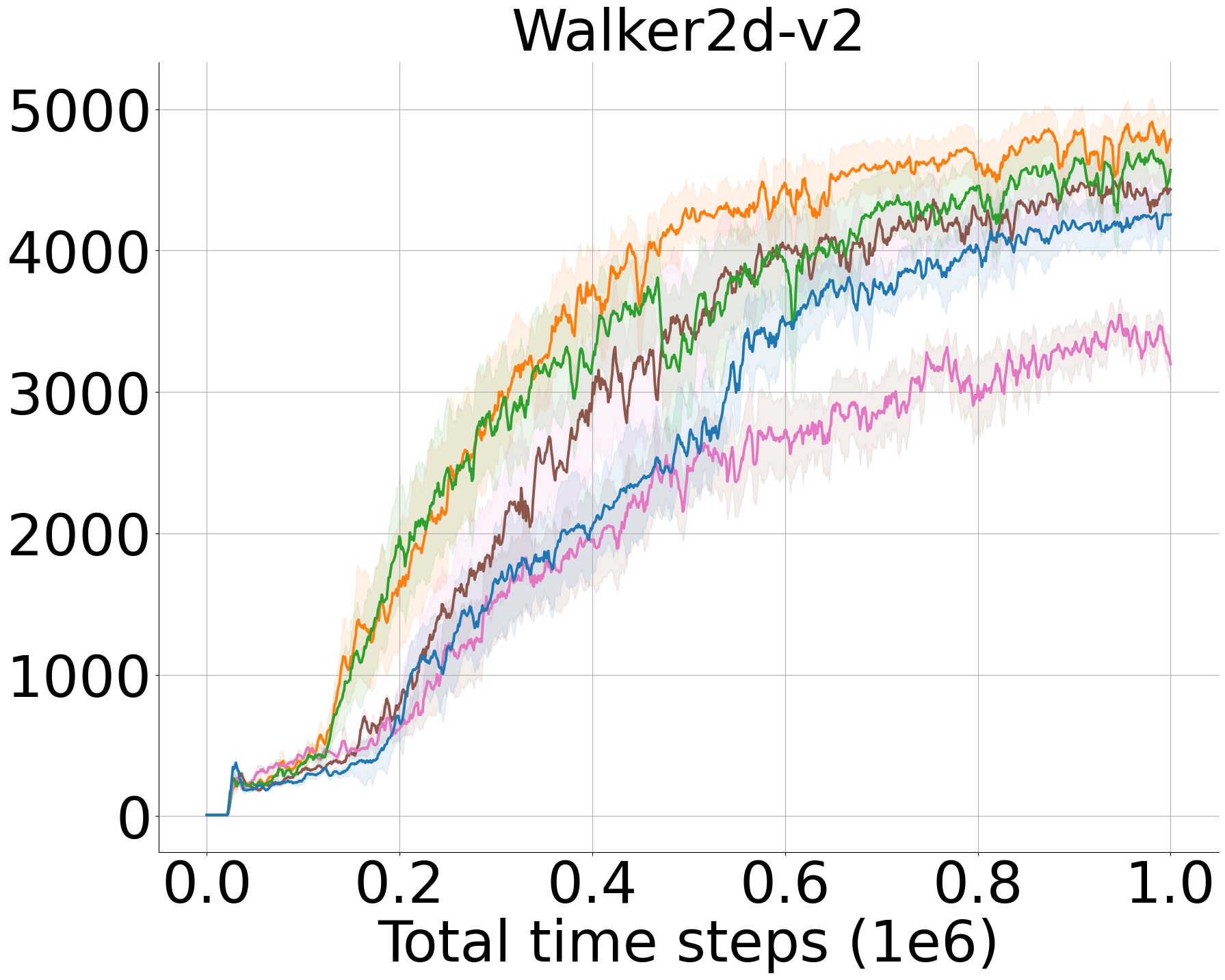}
	}
    \caption{Learning curves for the selected MuJoCo continuous control tasks, comparing ablation of LA3P under low TD error shared transitions, LA3P without the LAP function, LA3P without the PAL function, and LA3P without the shared set of transitions. Note that the TD3 algorithm is used as the baseline off-policy actor-critic algorithm. The shaded region represents a 95\% confidence interval over the trials. A sliding window of size 5 smoothes the curves for visual clarity.}
	\label{fig:ablation}
\end{figure*}

\begin{figure*}[!hbt]
    \centering
    \begin{align*}
        &\text{{\blue} $\lambda = 0.1$}  &&\text{{\brown} $\lambda = 0.3$} &&\text{{\orange} $\lambda = 0.5$} \\
        &\text{{\green} $\lambda = 0.7$} &&\text{{\pink} $\lambda = 0.9$}
    \end{align*}
	\subfigure{
		\includegraphics[width=1.55in, keepaspectratio]{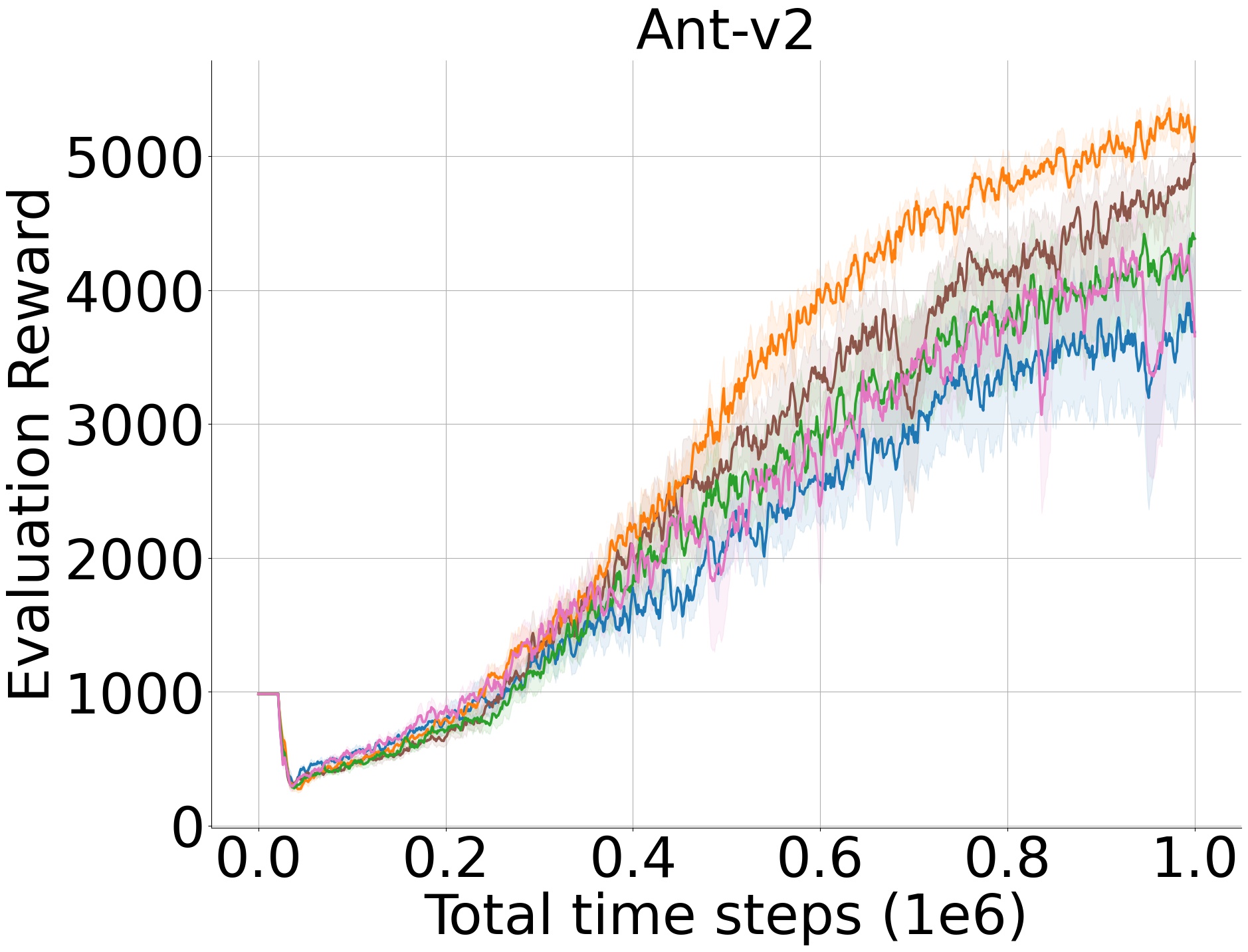}
		\includegraphics[width=1.55in, keepaspectratio]{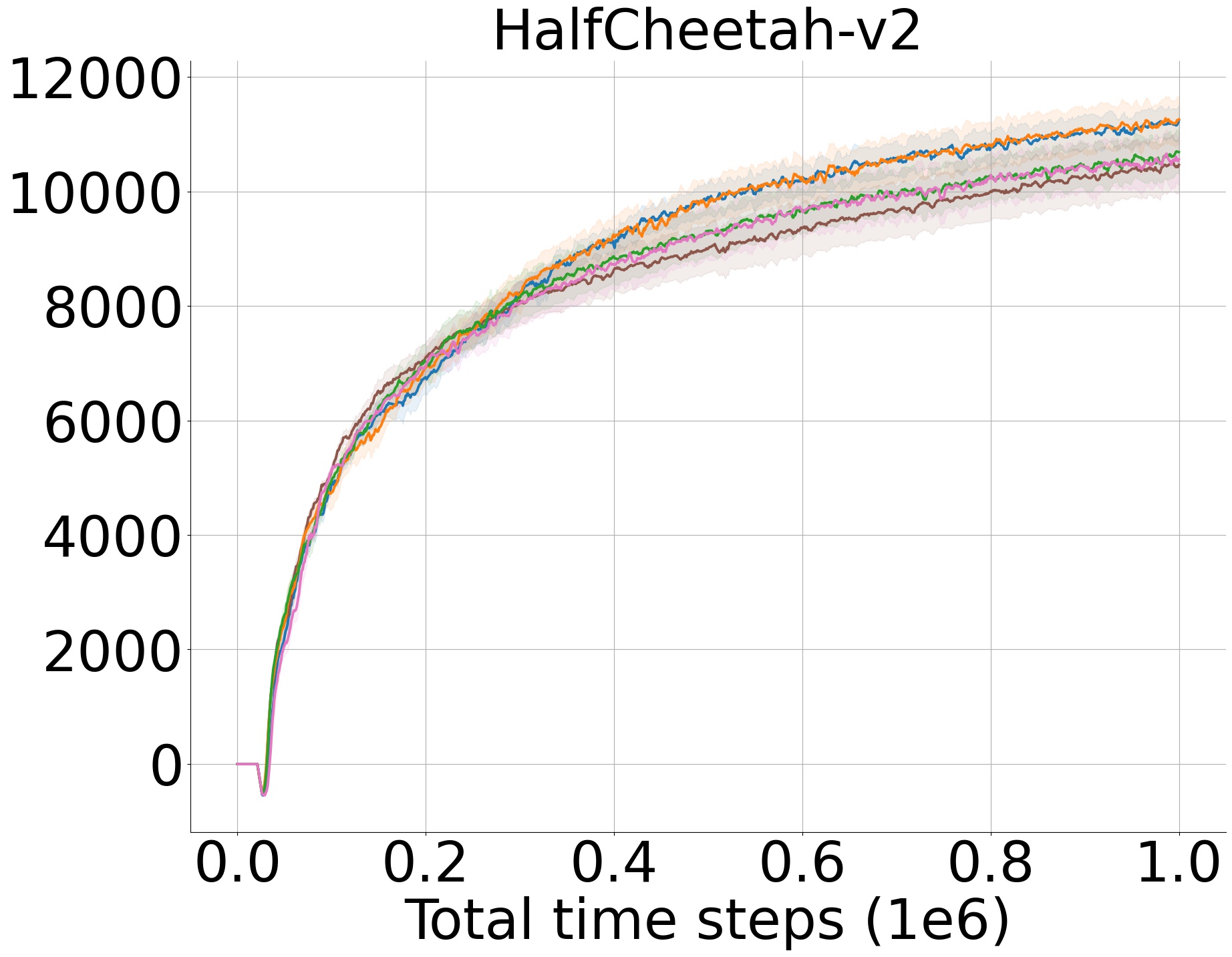}
		\includegraphics[width=1.55in, keepaspectratio]{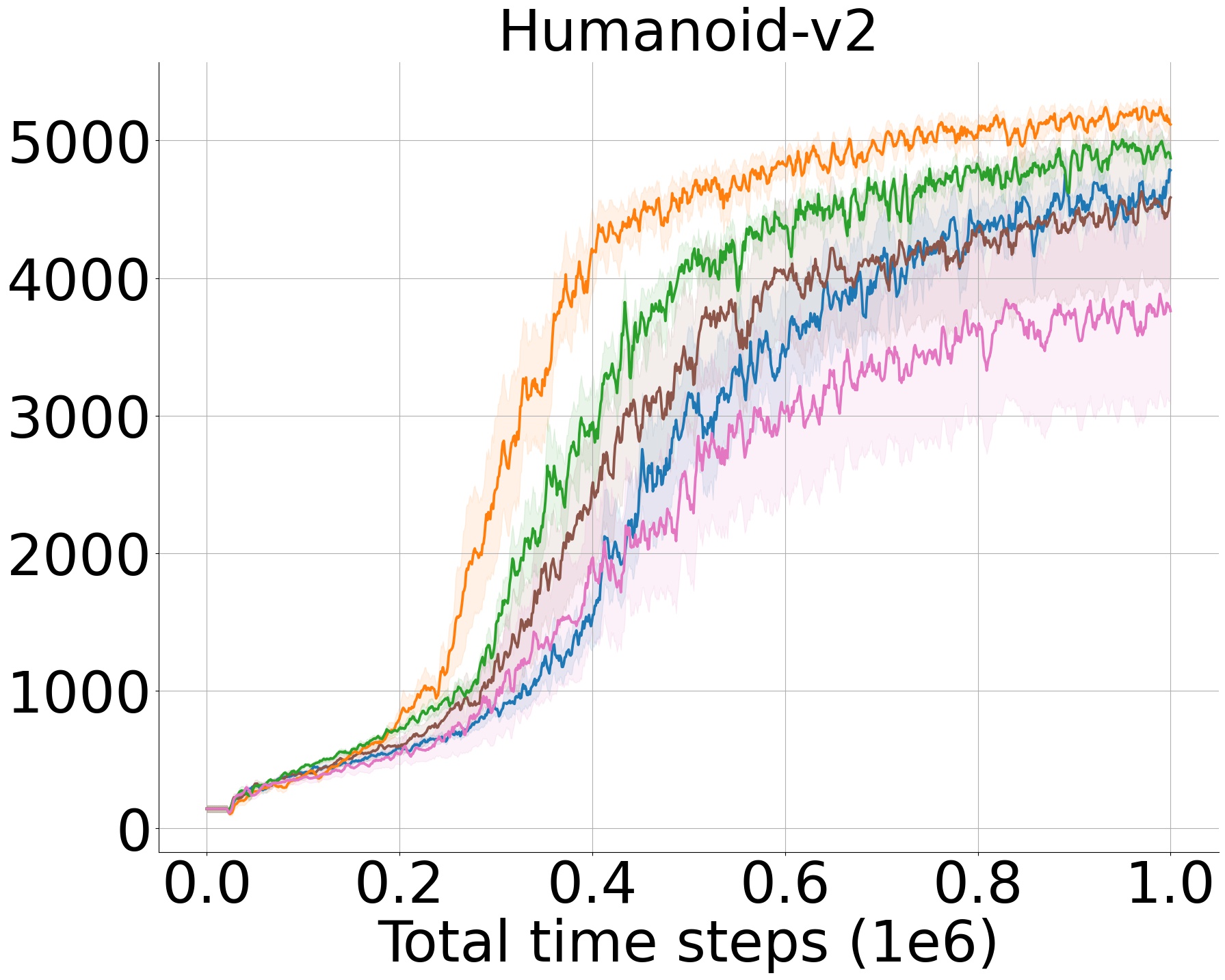}
		\includegraphics[width=1.55in, keepaspectratio]{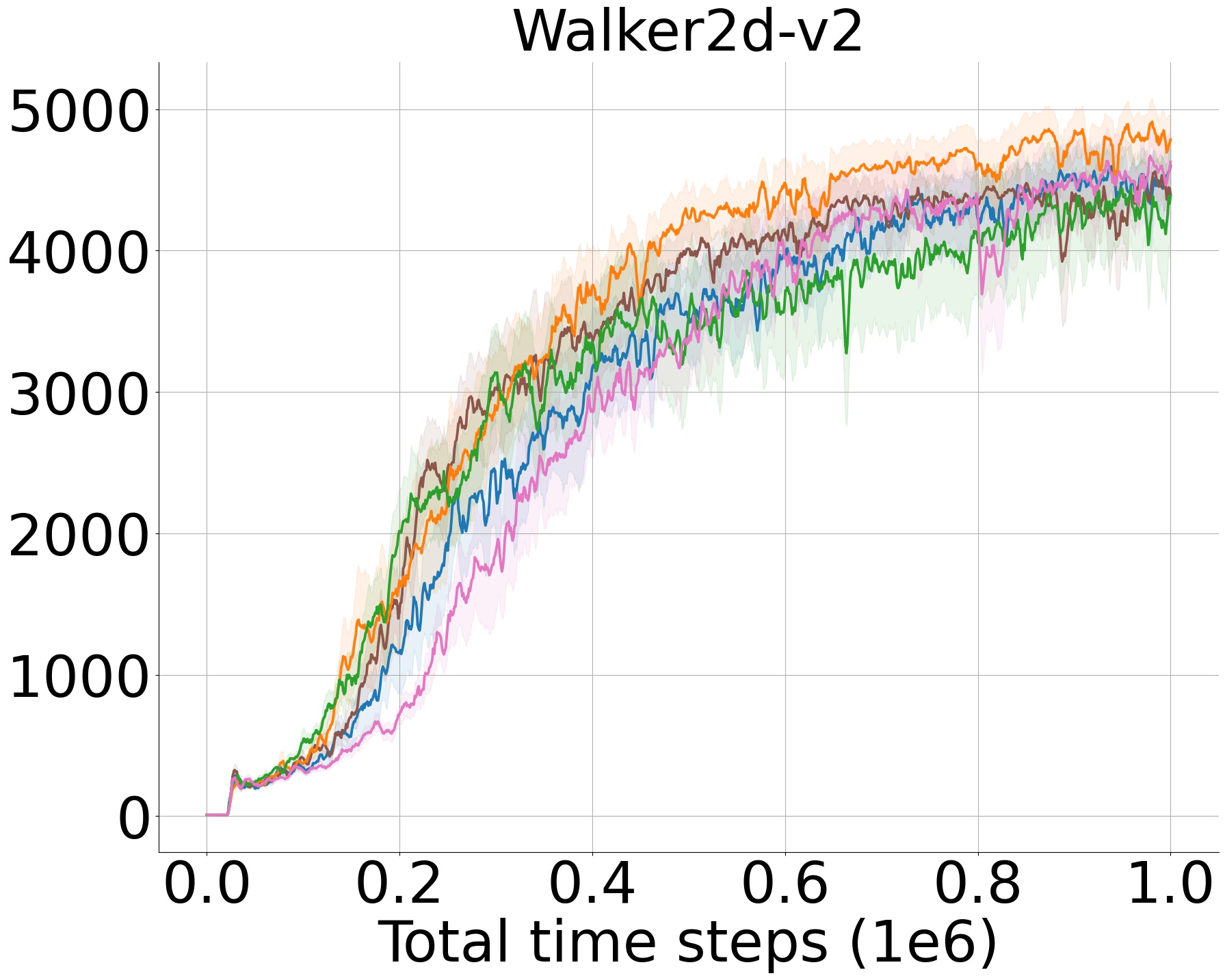}
	}
    \caption{Learning curves for the selected MuJoCo continuous control tasks, analyzing the sensitivity of LA3P with respect to $\lambda = \{0.1, 0.3, 0.5, 0.7, 0.9\}$. Note that the TD3 algorithm is used as the baseline off-policy actor-critic algorithm. The shaded region represents a 95\% confidence interval over the trials. A sliding window of size 5 smoothes the curves for visual clarity.}
	\label{fig:fraction}
\end{figure*}

First, we deduce that the set of shared transitions is the most crucial component of our framework. Independently training the actor and critic networks violate their correlation as the actor is optimized by maximizing the Q-values estimated by the Q-network, and the Q-network is trained using the actions selected by the actor. Thus, they should not be separated in training, and we empirically verify Observation \ref{cor:uniform}. Although the LAP and PAL functions apply the same number of transitions in each update step, i.e., $\lambda = 0.5$, we observe that the contribution of LAP is more significant than PAL. As discussed in our comparative evaluations, the performance improvement by LA3P largely relies on inverse sampling for the actor network. As expected, correcting the prioritization in inverse sampling for the actor network through the LAP approach, i.e., (\ref{inverse_eq_1}), is more crucial than correcting the loss by PAL for the uniformly sampled batch. Lastly, we infer that using low TD error transitions instead of uniformly sampled ones substantially degrades the performance. Although this setting would seem to be a reasonable choice at first glance, the data efficiency considerably decreases as the Q-network repeatedly trains with transitions that it has already learned well. In the expectation, the uniformly sampled batch of transitions corresponds to an intermediate TD error value compared to the transitions contained in the entire replay buffer. As we experimentally show, this may benefit both the actor and critic networks since inverse prioritized and prioritized sampling may not include transitions with intermediate TD error values compared to the rest of the experiences.

Our sensitivity analysis on the $\lambda$ parameter suggests that $\lambda = 0.5$ produces the best results by a notable margin in all environments. As $\lambda$ decreases, the correlation between the actor and critic networks starts to be ignored, and the performance drops. In contrast, the larger $\lambda$ values yield the performance to converge to that of uniform sampling. Hence, we believe the introduced framework does not require intensive hyper-parameter tuning, and $\lambda = 0.5$ can apply to many tasks. Overall, it is shown by our ablation studies that our framework improves over the baseline actor-critic algorithms due to the structure of the introduced method rather than unintended consequences or any exhaustive hyper-parameter tuning.

\section{Conclusion}
In this paper, we build the theoretical foundations behind the poor empirical performance of a widely known experience replay sampling scheme, Prioritized Experience Replay (PER) \citep{per}, for controlling continuous systems. To achieve this, we first show that transition tuples may exist so that the corresponding temporal-difference (TD) errors can increase the Q-value estimation error associated with the current or subsequent transition tuples. We use this finding to further indicate that training actor networks with large TD errors may cause the approximate policy gradient computed under the Q-network to diverge from the one computed under the optimal Q-function. This result suggests that even if the biased loss function in the PER algorithm is corrected, optimizing the actor network with low TD error transitions and Q-network with large TD error transitions can significantly increase the performance. This enables us to comprehend PER's poor performance in more detail when applied to continuous control algorithms.

However, training actor and critic networks with different transitions throughout the learning violates the actor-critic theory since each of them is optimized with respect to the other, i.e., the actor tries to maximize the Q-value estimated by the Q-network and Q-network computes its loss based on the actions selected by the actor. This allows us to develop a novel framework, Loss Adjusted Approximate Actor Prioritized Experience Replay (LA3P), which mixes training with uniformly sampled, low, and high TD error transitions. The introduced approach also accounts for the previous findings of \cite{lap}, which practically eliminate the outlier bias introduced by the combination of mean-squared error with PER. We test LA3P on standard deep reinforcement learning benchmarks in MuJoCo and Box2D and demonstrate that it substantially outperforms the competing methods and improves the state-of-the-art. An extensive set of ablation studies further implies that each LA3P component significantly impacts the offered performance improvement, and inverse prioritized sampling with corrected loss functions can increase the performance to the maximum. Therefore, we firmly believe that issues with PER in continuous control are corrected by the presented modifications supported by a comprehensive theoretical analysis. Finally, the source code for our algorithm is publicly available at our GitHub repository\footref{our_repo} for easy reproducibility and to support further research in non-uniform sampling methods.

\bibliographystyle{unsrtnat}
\bibliography{references}

\clearpage

\begin{appendices}
\section{Summary of the LA3P Framework}
\begin{figure}[htb]
\label{fig:simple_la3p}
        \centering
        
            \centering
            \begin{tabular}{c}
\begin{lstlisting}
uniform sampling
critic training with PAL
priority update
actor training
prioritized sampling
critic training with LAP
priority update
inverse prioritized sampling
actor training
\end{lstlisting}
             \end{tabular}
        \caption{A simplified depiction of the cascaded LA3P framework. Operations run consecutively.}
\end{figure}

\section{Experimental Details}
\label{app:exp_details}

\subsection{Architecture and Hyper-parameter Setting}
\subsubsection{Architecture}
The actor-critic methods, TD3 and SAC, employ two Q-networks and a single actor network. All networks feature two hidden layers having 256 hidden units, with ReLU activation functions after each. Following a final linear layer, the critic networks take state-action pairs $(s, a)$ as input and output a scalar value Q. The actor network takes state $s$ as input and produces a multi-dimensional action $a$ by applying a linear layer with a tanh activation function multiplied by the action space scale.

\subsubsection{Network Hyper-parameters}
The Adam optimizer \citep{adam} is used to train the networks, with a learning rate of $3 \times 10^{-4}$ and a mini-batch size of 256. After each update step, the target networks in both TD3 and SAC are updated using polyak averaging with $\zeta = 0.005$, resulting in $\theta^{\prime} \leftarrow 0.995 \cdot \theta^{\prime} + 0.005 \cdot \theta$.

\subsubsection{Terminal Transitions}
In setting the target Q-value, we utilize a discount factor of $\gamma = 0.99$ for non-terminal transitions and zero for terminal transitions. A transition is deemed terminal only if it stops due to a termination condition, i.e., failure or exceeding the time limit.

\subsubsection{Actor-Critic Algorithms}
We use the default policy noise of $\mathcal{N}(0, \sigma_{N})$ for the TD3 algorithm, as suggested by the author, where it is clipped to $[0.5, 0.5]$ with $\sigma_{N} = 0.2$. The range of the action space is used to scale both values. With SAC, we utilize the learned entropy variant, in which entropy is optimized to an objective of $-\texttt{action dimensions}$ using an Adam optimizer with a learning rate of $3 \times 10^{-4}$, similar to the actor and critic networks. To avoid numerical instability in the logarithm operation, we cut the log standard deviation to $(20, 2)$, and a small constant of $10^{-6}$ is added, as designated by the author. 

\subsubsection{Prioritized Sampling Algorithms}
As described by \cite{per}, we use $\alpha = 0.6$ and $\beta = 0.4$ for PER. As LAP and PAL functions are employed in our algorithm, we directly use $\alpha = 0.4$ and $\beta = 0.4$. No hyper-parameter optimization was performed on the $\alpha$ and $\beta$ parameters since the used values produce the best results, as reported by \cite{lap}. 

Since SAC and TD3 maintain two Q-networks, there are two TD errors defined by $\delta_{1} = y - Q_{\theta_{1}}$ and $\delta_{2} = y - Q_{\theta_{2}}$, where $y$ was previously defined in (\ref{eq:q_target}). Each priority considers the maximum of $\lvert\delta_{1}\rvert$ and $\lvert\delta_{2}\rvert$, as described by \cite{lap} to produce the strongest performance. New samples are assigned a priority equal to the highest priority $p_{\text{init}} = 1$ recorded at any time during learning, as done by PER.

Applications of LA3P to SAC and TD3 do not differ in terms of implementation and algorithmic setup. The main differences between the actor-critic algorithms of SAC and TD3 are the computation of the policy gradient, entropy tuning, and the presence of the target actor network. As discussed previously, Theorem \ref{thm:diverge_grad} is valid for deterministic and stochastic policies. Therefore, algorithmic differences between SAC and TD3 do not regard the implementation and operation of LA3P. 

\subsubsection{Exploration}
To fill the buffer, the agent is not trained for the first 25000 time steps, and actions are chosen randomly with uniform probability. After that, TD3 explores the action space by introducing a Gaussian noise of $\mathcal{N}(0, \sigma_{E}^{2} \cdot \texttt{max action size})$, where $\sigma_{E} = 0.1$ is scaled by the action space range. As SAC employs a stochastic policy, no exploration noise is added.

\subsubsection{Hyper-parameter Optimization}
No hyper-parameter optimization was performed on any algorithm except for SAC. Having the remaining parameters fixed, we optimized the reward scale for the BipedalWalker, LunarLanderContinuous, and Swimmer tasks, as they were not reported in the paper. We tested the values of $\{5, 10, 20\}$, and it turned out that scaling the rewards by 5 produced the best results for these environments. 

All algorithms follow what is reported in the original papers or the most recent code in the respective GitHub repositories. SAC follows the precise hyper-parameter setting outlined in the original paper except for increased exploration time steps to 25000 and entropy tuning. For TD3, as we employed the code from the author's repository\footref{td3_repo}, the parameter setting has a minor difference. Different from the original paper, the code in the repository increases the number of start steps to 25000 and batch size to 256 for all environments, as reported to produce better results. 

For LA3P, we tested $\lambda = \{0.1, 0.3, 0.5, 0.7, 0.9\}$ on the Ant, Hopper, Humanoid, and Walker2d tasks, and found that $\lambda = 0.5$ exhibited the best results. We provided the results under different $\lambda$ values in our ablation studies in Section \ref{sec:ab_stud}. For clarity, all hyper-parameters are presented in Table \ref{tab:hyper_params}.

\begin{table*}[!b]
\begin{center}
    \begin{tabular}{l c}
        \toprule
        \textbf{Hyper-parameter} & \textbf{Value} \\
        \midrule
        Optimizer & Adam \\
        Learning rate & $3 \times 10^{-4}$ \\
        Mini-batch size & 256 \\
        Discount factor $\gamma$ & 0.99 \\
        Target update rate & 0.005 \\
        Initial exploration steps & 25000 \\
        \midrule
        TD3 exploration policy $\sigma_{E}$ & 0.1 \\ 
        TD3 policy noise $\sigma_{N}$ & 0.2 \\ 
        TD3 policy noise clipping & $(-0.5, 0.5)$ \\ 
        \midrule
        SAC entropy target & \texttt{-action dimensions} \\
        SAC log-standard deviation clipping & $(-20, 2)$ \\
        SAC log constant & $10^{-6}$ \\ 
        SAC reward scale (except Humanoid) & 5 \\
        SAC reward scale (Humanoid) & 20 \\
        \midrule
        PER priority exponent $\alpha$ & 0.6 \\
        PER importance sampling exponent $\beta$ & 0.4 \\
        PER added priority constant & $10^{-4}$ \\ 
        LAP \& PAL exponent $\alpha$ & 0.4 \\
        LA3P uniform fraction $\lambda$ & 0.5 \\
        \bottomrule
    \end{tabular}
\end{center}
\caption{Hyper-parameters used in the experiments.}
\label{tab:hyper_params}
\end{table*}

\subsection{Implementation}
TD3 is implemented using the author's GitHub repository\footref{td3_repo}, while we utilize the code from the same author's LAP-PAL repository\footref{lap_repo} to implement PER and LAP-PAL functions. The PER implementation is based on proportional prioritization through a sum tree data structure. We manually implement SAC by following the original paper and adding entropy tuning the same authors introduced in \citep{sac_applications}. Lastly, We directly utilize the MaPER code from the paper's submission files from the OpenReview website\footref{maper_repo}. No changes were made to the MaPER code.

The implementation of LA3P consists of the cascaded uniform, prioritized, and inverse prioritized sampling, which precisely follows the pseudocode in Algorithm \ref{alg:la3p}. We do not update the priorities after the actor update with inverse prioritized sampling since the PER implementation with standard actor-critic algorithms only considers the priority update after each critic update. 

\subsection{Experimental Setup}
\subsubsection{Simulation Environments}
All agents are evaluated in continuous control benchmarks of MuJoCo\footnote{\url{https://mujoco.org/}} and Box2D\footnote{\url{https://box2d.org/}} physics engines interfaced by OpenAI Gym\footnote{\url{https://www.gymlibrary.ml/}}, using v2 environments. The environment, state-action spaces, and reward function are not altered or pre-processed for practical reproducibility and fair comparison with empirical findings. Each environment has a multi-dimensional action space with values ranging between $[1, 1]$, excluding Humanoid, which has a range of $[0.4, 0.4]$. 

\subsubsection{Evaluation}
Every 1000 time steps, an evaluation is performed, each being the average reward over 10 episodes, using the deterministic policy from TD3 without exploration noise or the deterministic mean action from SAC. We employ a new environment with a fixed seed (the training seed + a constant) for each evaluation to decrease the variation caused by varying seeds \citep{deep_rl_that_matters}, so each evaluation utilizes the same set of initial start states.

\subsubsection{Visualization of the Learning Curves}
Learning curves indicate performance and are depicted as an average of 10 trials with a shaded region denoting a 95\% confidence interval over the trials. The curves are flattened equally throughout a sliding window of 5 evaluations for visual clarity.

\section{Empirical Complexity Analysis}
\label{sec:emp_comp}
Upon completing our evaluation simulations, we compare the run time of baseline uniform sampling, PER, and our algorithm. Each sampling method combines the off-policy actor-critic algorithms, SAC and TD3. We record the total run time of each method throughout our comparative evaluation experiments. All experiments are run on a single GeForce RTX 2070 SUPER GPU and an AMD Ryzen 7 3700X 8-Core Processor. Our results are presented in Table \ref{tab:emp_comp}. 

\begin{table*}[!bh]
\begin{center}
    \begin{tabular}{@{} cl*{3}c @{}}
        \toprule
        & \textbf{Result} & \textbf{Uniform} & \textbf{PER} & \textbf{LA3P} \\
        \midrule
        & Run Time (mins) & 225.18 $\pm$ 1.48 & 307.01 $\pm$ 1.52 & 445.42 $\pm$ 1.52 \\
        \rot{\rlap{SAC}}
        & Time Increase (\%) & +0.00\% & +136.34\% & +197.81\% \\
        \midrule
        & Run Time (mins) & 131.51 $\pm$ 1.98 & 145.83 $\pm$ 2.09 & 238.29 $\pm$ 2.13 \\
        \rot{\rlap{TD3}}
        & Time Increase (\%) & +0.00\% & +110.89\% & +181.19\% \\
        \bottomrule
    \end{tabular}
\end{center}
\caption{Average run time of uniform sampling, PER, and LA3P, and their percentage increase over the off-policy actor-critic algorithms, SAC and TD3. Values are recorded over 1 million time steps and averaged over 10 random seeds and benchmark environments selected for our comparative evaluations. $\pm$ captures a 95\% confidence interval over the run time.
}
\label{tab:emp_comp}
\end{table*}

First, we find that the run time of SAC is greater than that of TD3. This is due to the additional entropy tuning that requires backpropagation and maintaining a stochastic actor. Moreover, the high dimensional environments such as Ant and Humanoid significantly increase the mean run time of the algorithms. While PER has a slightly increased run time, our method considerably increases the required time for each experiment. Although this result may question the feasibility of our approach, the empirical run time is vastly lower than what is indicated by our theoretical analysis. We know that $\mathcal{O}(\lvert R \rvert)$ is much larger than $\mathcal{O}(\log \lvert R \rvert)$, mainly when the replay buffer is large. Nonetheless, this can be overcome by the discussed parallelable array division operation embedded through SIMD operations in the recently introduced CPUs.

\end{appendices}

\end{document}